\documentclass[letterpaper]{article} 
\usepackage{aaai2026}
\usepackage{times}  
\usepackage{helvet}  
\usepackage{courier}  
\usepackage[hyphens]{url}  
\usepackage{graphicx} 
\usepackage{natbib}  
\usepackage{caption} 
\usepackage{soul}
\usepackage{amsmath}
\usepackage{amsfonts}
\usepackage{amsthm}
\usepackage{amssymb}
\usepackage{booktabs}
\usepackage{algorithm}
\usepackage{algorithmicx}
\usepackage[noend]{algpseudocode}
\usepackage{comment}
\usepackage{todonotes}
\usepackage{color, colortbl, etoolbox}
\usepackage{tabulary, array}
\usepackage{comment}
\usepackage{cancel}

\usepackage{array,multirow,makecell}
\usepackage[switch]{lineno}

\newtheorem{definition}{Definition}
\newtheorem{example}{Example}
\newtheorem{theorem}{Theorem}

\newtheorem{proposition}{Proposition}
 
\newtheorem{lemma}{Lemma}
\newtheorem{postulate}{Principle}

\newcommand{\cF}{{\mathcal{F}}}

\newcommand{\cL}{{\mathcal{L}}}

\newcommand{\bM}{{\mathbf{M}}}

\newcommand{\card}[1]{|#1|}
\newcommand{\tuple}[1]{\langle #1 \rangle}

\newcommand{\Arg}{{\mathtt{Arg}}(\cL)}
\newcommand{\ArgF}{{\mathtt{Arg}}(\cF)}

\newcommand{\supp}{\mathtt{S}}
\newcommand{\conc}{\mathtt{C}}

\newcommand{\CN}{\mathtt{CN}}

\newcommand{\CNF}{{\mathtt{CNF}}}

\newcommand{\uc}{{\sqcup}}

\newcommand{\agg}{g}

\newcommand{\similarity}{\mathtt{sim}}
\newcommand{\simFlatPos}{\similarity_{\text{ord}}^{\text{}}}
\newcommand{\simFlatUnord}{\similarity_{\text{unord}}^{\text{}}}
\newcommand{\simFlatPara}[1]{\similarity_{\text{para}}^{\text{}#1}}
\newcommand{\simi}{\similarity{\mathtt{Arg}}}
\newcommand{\simL}{\mathtt{sim{\mathtt{L}}}}
\newcommand{\simC}{\mathtt{sim{\mathtt{C}}}}
\newcommand{\simS}{\mathtt{sim{\mathtt{S}}}}
\newcommand{\simP}{\mathtt{sim{\mathtt{P}}}}

\newcommand{\setObj}{\mathbb{X}}

\newcommand{\avg}{\mathtt{avg}}

\newcommand{\jacc}{\mathtt{jac}}
\newcommand{\dice}{\mathtt{dic}}
\newcommand{\soren}{\mathtt{sor}}
\newcommand{\ander}{\mathtt{adb}}
\newcommand{\sok}{\mathtt{ss}}

\renewcommand{\min}{\mathtt{min}}
\renewcommand{\max}{\mathtt{max}}

\newcommand{\bm}{\mathtt{bm}}
\newcommand{\simsbert}{\mathtt{simP}^{\text{\tiny SBRT}}}
\newcommand{\best}[2]{\text{B}_{#1}^{#2}}  

\newcommand{\midsize}{\fontsize{9.5pt}{11pt}\selectfont}

\title{
Technical Report -- A Context-Sensitive Multi-Level Similarity Framework for First-Order Logic Arguments: An Axiomatic Study}
\author{
  Victor David$^{1}$, Jérôme Delobelle$^{2}$, 
Jean-Guy Mailly$^{3}$} 
\affiliations{
$^1$University Côte d'Azur, Inria, CNRS, I3S, Sophia-Antipolis, France \\
$^2$Université Paris Cité, LIPADE, Paris, France \\
$^3$Université Toulouse Capitole, IRIT, Toulouse, France\\
victor.david@inria.fr, jerome.delobelle@u-paris.fr, jean-guy.mailly@irit.fr}

\date{}

\begin{document}
\maketitle

\begin{abstract}
Similarity in formal argumentation has recently gained attention due to its significance in problems such as argument aggregation in semantics and enthymeme decoding. While existing approaches focus on propositional logic, we address the richer setting of First-Order Logic (FOL), where similarity must account for structured content. We introduce a comprehensive framework for FOL argument similarity, built upon: (1) an extended axiomatic foundation; (2) a four-level parametric model covering predicates, literals, clauses, and formulae similarity; (3) two model families, one syntax-sensitive via language models, both integrating contextual weights for nuanced and explainable similarity; and (4) formal constraints enforcing desirable properties.

\end{abstract}

\section{Introduction}

Formal argumentation typically involves two components: a representation module structuring arguments and their relations in a graph (support or attack), and a reasoning module determining the acceptability of arguments via a semantics. This paradigm is now widely explored in knowledge representation and reasoning, with applications in decision-making \cite{ZhongFLT19,LeturcB23}, XAI \cite{vcyras2021argumentative,guo2023argumentative}, judgmental forecasting \cite{irwin2022forecasting,Gorur0T23}, and enthymeme-based argumentation \cite{ben2024understanding,david2025logic}.

Across these diverse applications, a common concern has emerged: the need to assess not only whether arguments exist and interact, but also how similar they are. This notion of similarity, first formalized in the propositional setting~\cite{AmgoudD18}, has been explored through an axiomatic lens to ensure rational behavior~\cite{AmgoudD20,AmgoudD21}. One key motivation is to improve \emph{argument aggregation} in gradual semantics, where each argument receives a degree of acceptability. 
For example, consider three arguments: $A_1 = \langle \{p_1, p_1 \Rightarrow q\}, q \rangle$, $A_2 = \langle \{p_2, p_2 \Rightarrow q\}, q \rangle$, and $A_3 = \langle \{p_3, p_3 \Rightarrow r\}, r \rangle$. While $A_1$ and $A_2$ support the same claim $q$, $A_3$ supports an unrelated claim $r$. Suppose all three arguments attack a fourth argument $B = \langle \{\neg q \land \neg r\}, \neg q \land \neg r \rangle$. In gradual semantics, the degree of acceptability of $B$ is influenced by its attackers. If similarity is not taken into account, $A_1$ and $A_2$ may overly lower the acceptability of $B$ due to overlapping information. 
In contrast, a similarity-aware approach would recognize the overlap between $A_1$ and $A_2$, down-weighting their combined effect, while treating the impact of $A_3$ independently. This prevents redundancy, leading to more accurate evaluations.

More recently, similarity has also been used to evaluate enthymeme decoding~\cite{ben2024axiomatic}, where the goal is to reconstruct incomplete arguments. In that setting, similarity helps quantify how well the decoding preserves the original information.

While most existing approaches focus on propositional logic, only few works have considered First-Order Logic (FOL) \cite{david2023similarity,david2025similarity}.  
However, these approaches struggle with syntactic proximity, and face scalability issues due to variable instantiation. To overcome these limitations, we propose a new framework for FOL argument similarity that (i) captures syntactic sensitivity via language-model-based measures, (ii) avoids instantiation by approximating variables as constants, and (iii) supports context-sensitive weighting on symbols and clauses.

Our framework defines a four-level similarity model, from predicates to full formulae, with theoretical guarantees and high potential for NLP applications such as text similarity, or argument clustering. 
With recent progress in text-to-logic translation~\cite{han2022folio,lu2022parsing,yang2023harnessing,lalwani2024nl2fol,ryu2024divide,lee2025entailment}, applying our hybrid similarity models, combining logical structure with multi-level semantic evaluation, appears increasingly promising for natural language arguments, though their full potential remains an open research direction.

\section{Logic and Arguments}

\subsection{First-Order Logic and Arguments}

We consider a standard first-order logic (FOL) language $\cL$, built from a set of predicate symbols $P, Q, \dots$ and terms constructed from  constants, variables, and function symbols. A formula $\phi$ is built using the standard logical connectives ($\neg, \wedge, \vee, \Rightarrow, \Leftrightarrow$) and quantifiers ($\forall, \exists$).

We write $\Phi \vdash \alpha$ for deduction, and denote the deductive closure of $\Phi$ by $\CN(\Phi) = \{\alpha \mid \Phi \vdash \alpha\}$.
A set of formulae $\Phi$ is inconsistent if $\Phi \vdash \bot$.  
Following \cite{besnard2005practical}, a FOL argument is a minimal set of formulae sufficient to deduce a claim.

\begin{definition}\label{def:argument}
\upshape{
A FOL \textbf{argument} is a pair $A = \langle \Phi, \alpha \rangle$ s.t. $\Phi \subset \cL$, $\alpha \in \cL$ and: (1: consistency) $\Phi \not\vdash \bot$, (2: inference) $\Phi \vdash \alpha$, and (3: minimality) no $\Phi' \subset \Phi$ satisfies $\Phi' \vdash \alpha$. 
We write $\supp(A) = \Phi$ (support), $\conc(A) = \{\alpha\}$ (claim), and denote the set of all FOL arguments by $\Arg$. An argument is trivial if $\supp(A) = \emptyset$ and $\conc(A) = \top$.

}
\end{definition}

As discussed in \cite{AmgoudD18}, 
two finite sets of formulae $\Phi, \Psi \subseteq \cL$ are \emph{equivalent}, i.e., $\Phi \cong \Psi$, iff  there is a bijection ${f: \Phi \rightarrow \Psi}$ s.t. $\forall \phi \in \Phi$, \linebreak $\phi \equiv f(\phi)$. 
We use this restricted equivalence instead of \linebreak $\CN(\Phi) = \CN(\Psi)$ to prevent false matches from incorrect \linebreak information.
For example the sets {\small $\{\texttt{Square(a)}, \linebreak  \texttt{Square(a)} \Rightarrow \texttt{Rectangle(a)}\}$} and {\small $\{\texttt{Rectangle(a)}, \linebreak \texttt{Rectangle(a)} \Rightarrow  \texttt{Square(a)}\}$} should not be equivalent. 

We use pairwise equivalence between formula sets to define a bijective equivalence between FOL arguments.

\begin{definition}\label{def:bijeq}
\upshape{
Two FOL arguments $A, B \in \Arg$ are \textbf{bijectively equivalent}, denoted $A \approx B$, iff there exist bijections:
(1) $f: \supp(A) \rightarrow \supp(B)$ with $\forall \phi \in \supp(A),~ \phi \equiv f(\phi)$; (2) $f': \conc(A) \rightarrow \conc(B)$ with $\forall \psi \in \conc(A),~ \psi \equiv f'(\psi)$.
}
\end{definition}

\subsection{CNF Transformation}\label{sec:CNF}

Throughout the paper, we assume that $\cL$ is a first-order language over a finite signature, i.e., with a finite set of constants, functions, and predicates. This restriction ensures that the space of CNF formulae remains bounded and comparable, which is reasonable in our context, as formulae originate from natural language with a limited vocabulary.

A formula $\phi$ is in \textbf{conjunctive normal form} (CNF) if it is a conjunction of clauses $\phi = \bigwedge_i C_i$, each $C_i$ being a disjunction of literals. 
Following \cite{amgoud2021similarity}, we map each formula $\phi \in \cL$ to a compiled CNF formula $\CNF(\phi) \in \cF \subset \cL$, where $\cF$ is a fixed finite sublanguage containing only quantifier-free formulae in conjunctive normal form (CNF), possibly obtained after Skolemization. The CNF transformation involves standard steps: eliminating implications and biconditionals, pushing negations inward (De Morgan), standardizing variables, 
applying Skolemization to eliminate existential quantifiers with Skolem functions or constants, dropping universal quantifiers, distributing disjunctions over conjunctions, and applying deterministic syntactic normalization (e.g., ordering literals and clauses). 

\begin{example}
Consider the sentence:  
\emph{``Every dog loves some bone."}  
Its FOL form is:  
$\forall x.~ \texttt{Dog}(x) \Rightarrow \exists y.~ (\texttt{Bone}(y) \wedge \texttt{Loves}(x, y))$.  
CNF compilation yields:  
$\{\neg \texttt{Dog}(x) \vee \texttt{Bone}(f(x)),\ 
\neg \texttt{Dog}(x) \vee \texttt{Loves}(x, f(x))\}$.
\end{example}

This transformation is not unique, as CNF representations may vary with normalization or Skolem naming. 
We do not fix a strategy, but assume that one method \(\CNF\) is chosen.

\textbf{Important note:} 
Although $\CNF(\phi)$ is not always logically equivalent to $\phi$, since Skolemization preserves only satisfiability, it provides a canonical  approximation, well-suited for structural comparison in our similarity framework. 

This approximation is a deliberate and practical choice. Skolemization handles quantifier-heavy formulae by replacing existential quantifiers with terms. This is useful in natural language settings, where grounding is infeasible due to vast or undefined domains (e.g., all people, all locations), avoiding the combinatorial explosion of full instantiation.

By contrast, the Order-Sorted (OS) FOL framework proposed in~\cite{david2023similarity} preserves equivalence more precisely, but becomes impractical in large or open domains. Our approximation thus favors scalability and operational feasibility, while preserving sufficient structure for fine-grained similarity comparison.

\begin{definition}
\upshape{
For a finite set of formulae $\Phi \subseteq \cL$, the \textbf{CNF compiled set} is:
$\uc(\Phi) = \bigcup_{\phi \in \Phi} \CNF(\phi)$, 
with variable renaming to avoid clashes.
}
\end{definition}

We extend the classical FOL argument definition to a CNF-based form, where all components are in CNF and the claim is a set of clauses. While semantically equivalent, this structure enables finer analysis, such as minimality.

\begin{definition}
\upshape{
A \textbf{CNF argument} is a pair $A = \langle \Phi, \Psi \rangle$ such that $\Phi \subseteq \cF$, $\Psi \subseteq \cF$ and:  
(1) $\Phi \not\vdash \bot$, (2) $\CN(\Psi) \subseteq \CN(\Phi)$, and (3) no $\Phi' \subset \Phi$ satisfies $\Phi' \vdash \Psi$. 
We denote the universe of all CNF arguments by $\ArgF$ (included in $\Arg$).
}
\end{definition}

We now define the compiled version of an argument.

\begin{definition}
\upshape{
A \textbf{compiled argument} $A^c = \langle \uc(\Phi), \uc(\alpha) \rangle$ is obtained by compiling both the support and the claim of $A = \langle \Phi, \alpha \rangle \in \Arg$.
}
\end{definition}

Some arguments in $\Arg$ may compile into forms that do not satisfy the constraint (3) of a CNF argument. For example, $A = \langle \{P(a) \wedge P(b)\}, P(a) \rangle \in \Arg$, but $A^{\mathtt{c}} = \langle \{P(a), P(b)\}, \{P(a)\} \rangle \notin \ArgF$. As shown in~\cite{amgoud2021similarity}, clausal argument ensures minimality.

\section{Multi-Level Arguments Similarity}\label{sec:sim-model}

A similarity measure indicates to what extent two objects (e.g., predicates, clauses, or arguments) share common features, such as semantic meaning, or syntactic structure. 

\begin{definition}\label{sim-measure}
\upshape{
Let $\setObj$ be a set of objects. A \textbf{similarity measure} on $\setObj$ is a function $\mathtt{simX} : \setObj \times \setObj \rightarrow [0,1]$, with $1$ for maximum similarity and $0$ for none.

}
\end{definition}

We propose a multi-level similarity model for CNF formulae in FOL, structured around four levels of abstraction:
Predicate and term similarity ($\simP$), forming the foundation; 
Literal similarity ($\simL$), which depends on $\simP$;
Clause similarity ($\simC$), which depends on $\simL$;
Set-of-clauses similarity ($\simS$), which depends on $\simC$.
All these levels are instantiated in the following sections. 
Each level builds on the one below, enabling structured and explainable computations.
We call any tuple $\bM = \tuple{\simP, \simL, \simC, \simS}$ a \textbf{similarity model}.

Similarity between two CNF arguments is computed separately on premises and claims using our multi-level method, then aggregated with weights reflecting their importance, as in~\cite{AmgoudD18}.

\begin{definition}
\label{def:smArguments}
\upshape{
Let a factor $\eta \in (0,1)$, and a similarity model $\bM = \tuple{\simP, \simL, \simC, \simS}$. 
Let $A,B \in \Arg$, 
we define the \textbf{similarity measure between FOL arguments} from $\Arg \times \Arg$ to $[0,1]$ by $\similarity\mathtt{Arg}^{\bM}_{\eta}(A, B) =$
\begin{center}
$ \eta \times \simS(\supp(A^{\mathtt{c}}), \supp(B^{\mathtt{c}}))~ + $
$(1-\eta) \times \simS(\conc(A^{\mathtt{c}}), \conc(B^{\mathtt{c}})).$ 
\end{center}
}
\end{definition}

\section{Axiomatic Foundations}\label{subsection:background-similarity}

Principles that argument similarity measures should satisfy have been discussed in \cite{AmgoudD18,AmgoudDD19}. 
Some of the principles 
can be stated exactly as in \cite{amgoud2021similarity}, since they do not concern the internal structure of the arguments. 

\begin{postulate} \upshape{
A similarity measure $\simi$ satisfies \textbf{Maximality} iff 
$\forall A \in \Arg$, $\simi(A,A) = 1$.
}
\end{postulate}

\begin{postulate} \upshape{
A similarity measure $\simi$ satisfies \textbf{Substitution} iff 
$\forall A,B,C \in \Arg$, if $\simi(A,B) = 1$ then $\simi(A,C) = \simi(B,C)$.
}
\end{postulate}

\begin{postulate} \upshape{
A similarity measure $\simi$ satisfies \textbf{Symmetry} iff 
$\forall A,B \in \Arg$, $\simi(A,B) = \simi(B,A)$.
}
\end{postulate}

This list includes desirable, but not always required, principles. For instance, some argue that symmetry is not essential for similarity \cite{Tversky77,Jantke94}, and \cite{david2023similarity} explores non-symmetric measures for OS-FOL arguments.

For content-related principles, we propose the following adaptations in the FOL setting using a similarity model.

The principle Syntax Independence states that similarity should ignore syntactic variation and rely only on structure.

\begin{definition}
\label{def:isomorphic_arguments}
\upshape{
Let \( A = \langle \Phi, \alpha \rangle, B = \langle \Phi', \alpha' \rangle \in \Arg \). 
Let \( \pi \) be a bijective renaming function over predicate names and term symbols (constants and variables), extended homomorphically to formulae and sets of formulae. 
We say that \( A \) and \( B \) are \textbf{isomorphic}, if:
$\pi(\Phi) = \Phi'  \text{ and }  \pi(\alpha) = \alpha'$.
}
\end{definition}

\begin{postulate} \upshape{
A similarity measure $\simi$ satisfies \textbf{Syntax Independence} iff for any bijective renaming function $\pi$, 
$\forall A,B,A',B' \in \Arg$, s.t. $A$ (resp. $B$) and $A'$ (resp. $B'$) are isomorphic wrt $\pi$, it holds that $\simi(A,B) = \simi(A',B')$.
}
\end{postulate}

Unlike prior work in propositional logic \cite{AmgoudD18}, where principles were defined over syntactic structures, we reconsider the need for Syntax Independence. In our framework, predicate and constant names carry importance, making similarity intentionally syntax-dependent. 
This choice is motivated by practical needs: in real applications, semantically related symbols (e.g., \texttt{house} and \texttt{apartment}) should be closer than unrelated ones (e.g., \texttt{house} and \texttt{banana}), despite differing syntax.

In the following content-sensitive principles, we focus on compiled arguments which have no irrelevant information, ensuring reliable similarity evaluation. We now reformulate the Minimality and Non-Zero principles, which ensure that overall argument similarity arises directly from predicate and term-level similarity. 
Minimality states that no argument similarity should arise without such base-level similarity.

The first condition excludes the case where both arguments have an empty support and so no intersection to compare, while the second and third conditions ensure that no predicate and term similarity appears in the supports or claims of the arguments, respectively.

\begin{postulate} \upshape{
A similarity measure $\simi$ satisfies \textbf{Minimality} iff 
$\forall A, B \in \Arg$ s.t. $A^\mathtt{c}, B^\mathtt{c} \in \ArgF$, if: 
{\small
\begin{enumerate}
    \item $A^\mathtt{c}$ and $B^\mathtt{c}$ are not trivial,
    \item $\forall C_1 \in \supp(A^\mathtt{c}), \forall (P_1, \vec{a_1}) \in C_1, \forall a_1^i \in \vec{a_1}, 
    \forall C_2 \in \supp(B^\mathtt{c}), \forall (P_2, \vec{a_2}) \in C_2, \forall a_2^j \in \vec{a_2}$, $\simP(P_1,P_2) = 0$ and $\simP(a^i_1,a^j_2) = 0$,
    \item $\forall C_1 \in \conc(A^\mathtt{c}), \forall (P_1, \vec{a_1}) \in C_1, \forall a_1^i \in \vec{a_1}, 
    \forall C_2 \in \conc(B^\mathtt{c}), \forall (P_2, \vec{a_2}) \in C_2, \forall a_2^j \in \vec{a_2}$, $\simP(P_1,P_2) = 0$ and $\simP(a^i_1,a^j_2) = 0$,
\end{enumerate}
}
then $\simi(A,B) = 0$.
}
\end{postulate}

Conversely, Non-Zero \cite{David21} ensures that if similarity exists in the supports (here at the predicate and term level), it must propagate and result in some degree of similarity between arguments. 
\begin{postulate} \upshape{
A similarity measure $\simi$ satisfies \textbf{Non-Zero} iff 
$\forall A, B \in \Arg$ s.t. $A^\mathtt{c}, B^\mathtt{c} \in \ArgF$, if: 
{\small
\begin{enumerate}
    \item $A^\mathtt{c}$ and $B^\mathtt{c}$ are not trivial,
    \item $\exists C_1 \in \supp(A^\mathtt{c})$ and $\exists C_2 \in \supp(B^\mathtt{c})$ s.t. $\exists (P_1, \vec{a_1}) \in C_1$ and $\exists (P_2, \vec{a_2}) \in C_2$ s.t. either $\simP(P_1,P_2) > 0$ or there exists a position $i$, where $a^i_1 \in \vec{a_1}$ and $a^i_2 \in \vec{a_2}$ s.t. $\simP(a^i_1,a^i_2) > 0$, 
\end{enumerate}
}
then $\simi(A,B) > 0$.
}
\end{postulate}

In the following principles, we aim to study the impact of similarity between arguments at the intermediate level between clauses. More specifically, we focus on how the aggregation of clause-level similarity behaves.

We define next Monotony principles to capture how adding new information to one argument affects its similarity to a third. S-Monotony (S for support) generalizes classic Monotony, while C-Monotony (C for claim) extends the Dominance principle \cite{AmgoudD18}.

The first condition, defines that argument $B$ has the same support (resp. claim) as argument $A$, with one additional clause. Condition 2 ensures that the similarity over the remaining component (claim or support) is equal, allowing us to isolate the effect of similarity on a single component of the argument.
We then examine extreme cases: adding a clause with no similarity (score 0) must not increase the similarity between arguments, while adding a clause with full similarity (score 1) must not decrease it.

\begin{postulate} \label{prin-mono} 
\upshape{ 
A similarity measure $\simi$ satisfies \textbf{Monotony} iff $\forall A, B, C \in \Arg$ such that $A^\mathtt{c}, B^\mathtt{c}, C^\mathtt{c} \in \ArgF$, if: 
{\small
\begin{enumerate}
    \item $\supp(B^\mathtt{c}) = \supp(A^\mathtt{c}) \cup \{\beta\}$ (where $\beta$ is a new clause),
    \item $\simS(\conc(A^\mathtt{c}),\conc(C^\mathtt{c})) = \simS(\conc(B^\mathtt{c}),\conc(C^\mathtt{c}))$; and either 
\end{enumerate}
}
$\bullet$~(\textbf{S-Monotony$^0$}) if $\forall \alpha \in \supp(C^\mathtt{c})$, $\simC(\alpha,\beta) = 0$, then $\simi(A,C) \geq \simi(B,C)$; or 

$\bullet$~(\textbf{S-Monotony$^1$}) if $\exists \alpha \in \supp(C^\mathtt{c})$ s.t.  $\simC(\alpha,\beta) = 1$, then $\simi(A,C) \leq \simi(B,C)$.

{\small
\begin{enumerate}
    \item $\conc(B^\mathtt{c}) = \conc(A^\mathtt{c}) \cup \{\beta\}$ (where $\beta$ is a new clause),
    \item $\simS(\supp(A^\mathtt{c}),\supp(C^\mathtt{c})) = \simS(\supp(B^\mathtt{c}),\supp(C^\mathtt{c}))$; and either 
\end{enumerate}
}
$\bullet$~(\textbf{C-Monotony$^0$}) if $\forall \alpha \in \conc(C^\mathtt{c})$, $\simC(\alpha,\beta) = 0$, then $\simi(A,C) \geq \simi(B,C)$; or 

$\bullet$~(\textbf{C-Monotony$^1$}) if $\exists \alpha \in \conc(C^\mathtt{c})$ s.t. $\simC(\alpha,\beta) = 1$, then $\simi(A,C) \leq \simi(B,C)$.
}
 \end{postulate}

We propose also a new family of principles, Reinforcement, that formalize and constrain the impact of distinct pieces of information between two arguments that are otherwise similar. 
Conditions 1 and 2 ensure that $A$ and $B$ differ by exactly one clause, and Condition 3 focuses on a single component (support or claim). We define two versions of the Reinforcement principle. In the first (Reinforcement$^\geq$), if, for every clause in a component of $C$, the differing clause of $A$ is at least as similar to it as that of $B$, then $A$ must be at least as similar to $C$ as $B$ is. In the second (Reinforcement$^>$), if the differing clause of $A$ is strictly more similar to every clause in $C$ than that of $B$, then $A$ must be strictly more similar to $C$ than $B$ is.

\begin{postulate} \label{prin-reinf} 
\upshape{ 
A similarity measure $\simi$ satisfies \textbf{Reinforcement} iff $\forall A, B, C \in \Arg$ such that $A^\mathtt{c}, B^\mathtt{c}, C^\mathtt{c} \in \ArgF$, if: 
{\small
\begin{enumerate}
    \item $\supp(A^\mathtt{c}) \setminus \supp(B^\mathtt{c}) = \{\alpha\}$,
    \item $\supp(B^\mathtt{c}) \setminus \supp(A^\mathtt{c}) = \{\beta\}$,
    \item $\simS(\conc(A^\mathtt{c}),\conc(C^\mathtt{c})) = \simS(\conc(B^\mathtt{c}),\conc(C^\mathtt{c}))$; and either 
\end{enumerate}
}
$\bullet$~(\textbf{S-Reinforcement$^\geq$}) if $\forall \phi \in \supp(C^\mathtt{c})$, $\simC(\alpha,\phi) \geq \simC(\beta,\phi)$, then $\simi(A,C) \geq \simi(B,C)$; or 

$\bullet$~(\textbf{S-Reinforcement$^>$}) if $\forall \alpha \in \supp(C^\mathtt{c})$, $\simC(\alpha,\phi) > \simC(\beta,\phi)$, then $\simi(A,C) > \simi(B,C)$.

{\small
\begin{enumerate}
    \item $\conc(A^\mathtt{c}) \setminus \conc(B^\mathtt{c}) = \{\alpha\}$,
    \item $\conc(B^\mathtt{c}) \setminus \conc(A^\mathtt{c}) = \{\beta\}$,
    \item $\simS(\supp(A^\mathtt{c}),\supp(C^\mathtt{c})) = \simS(\supp(B^\mathtt{c}),\supp(C^\mathtt{c}))$; and either 
\end{enumerate}
}
$\bullet$~(\textbf{C-Reinforcement$^\geq$}) if $\forall \phi \in \conc(C^\mathtt{c})$, $\simC(\alpha,\phi) \geq \simC(\beta,\phi)$, then $\simi(A,C) \geq \simi(B,C)$; or 

$\bullet$~(\textbf{C-Reinforcement$^>$}) if $\forall \alpha \in \conc(C^\mathtt{c})$, $\simC(\alpha,\phi) > \simC(\beta,\phi)$, then $\simi(A,C) > \simi(B,C)$.
}
 \end{postulate}

\section{Context-Sensitive Multi-Level Similarity}\label{section:context-sensitive}

Motivated by practical applications involving natural language argumentation, we propose a three-stage pipeline that bridges symbolic representation, contextual weighting, and structured similarity computation: 

    \textbf{(I) Logical Transformation.} This step maps input texts to CNF-FOL representations, making argument structure explicit and enabling comparison over predicates, terms, and clauses.  
    We use ChatGPT-4o for CNF generation, but automating and evaluating this step is left for future work.

    \textbf{(II) Weight Assignment.} 
    This stage assigns contextual weights to logical components, allowing the model to prioritize informative elements for similarity.  
    For instance, action predicates (e.g., \texttt{Tease}) may matter more than descriptive ones (e.g., \texttt{HasColor}), and agent constants (e.g., \texttt{dog}) more than peripheral details (e.g., \texttt{brown}). 
    Weights are manually set in our examples; automating their learning is left for future work.

    \textbf{(III) Similarity Computation.} The similarity between two CNF arguments is computed in four steps:\\
        \textbf{1. Unweighted Clause Alignment:} 
        Each clause from a premise or claim is matched to its best semantic match in the other argument, using a flat literal similarity measure. \\
        \textbf{2. Weighted Clause Similarity:} Clause-pair scores are refined using the importance weights assigned to predicates, terms, and clauses, producing a context-aware similarity. \\
        \textbf{3. Formulae Similarity:} Clause similarities are aggregated into formulae-level scores using an aggregation.\\
        \textbf{4. Argument Similarity:} Computed from the aggregated formulae scores (see Definition~\ref{def:smArguments}).

This pipeline combines scalable learning-based models with symbolic representations to compute explainable, context- and syntax-aware similarity scores, offering theoretical guarantees and adaptability to natural language.

Like argument equivalence (Definition~\ref{def:bijeq}), based on local matchings between formulae, our similarity model compares argument components by identifying best-matching clause pairs. To apply this idea, we first use flat clause similarity to select matches based on intrinsic similarity. Then, importance weights refine these matchings at two levels: predicate and term weights adjust literal-level similarity, while clause weights influence formulae-level similarity, integrating contextual relevance into the overall similarity.

\begin{example}
    In the next sub-sections, we illustrate our definitions by computing the similarity between two texts from the STSb dataset \cite{cer2017semeval}:\\
T1: \textit{``A dog is teasing a monkey at the zoo"};\\ 
T2: \textit{``A monkey is teasing a dog at the zoo"}. \\
As syntax-sensitive functions benefit from full predicate names (e.g., \texttt{AtLocation}), we use them in practice, but shorten them (e.g., \texttt{AtLoc}) for readability.

\noindent CNF\_T1 = {\small $\{\texttt{AtLoc(monkey, zoo)}, \texttt{AtLoc(dog, zoo)},\\ \texttt{Tease(dog, monkey)}\}$};\\
CNF\_T2 = {\small $\{\texttt{AtLoc(monkey, zoo)}, \texttt{AtLoc(dog, zoo)}, \\ \texttt{Tease(monkey, dog)}\}$}.
\end{example}

\subsection{Step 1: Unweighted Best Clause Matching}
To better capture the semantic contribution of polarity, we incorporate it directly into the predicate symbol. For instance, the literal \(\neg \texttt{P}(a, b)\) is rewritten as \(\texttt{notP}(a, b)\). This transformation allows the similarity function to distinguish between positive and negative predicate forms, rather than treating polarity as a separate feature. It also enables the identification of logical correspondences between concepts with opposite polarity, such as \(\texttt{notEven}(x)\) and \(\texttt{Odd}(x)\), which may convey equivalent meanings.

In the examples, we will use SBERT \cite{reimers2019sentence} (\texttt{all-MiniLM-L6-v2}) to compute semantic similarity between predicates (with polarity) and parameters, denoted by $\simsbert$. While newer models (e.g., \cite{huang2024cosent}) offer slight improvements, SBERT provides strong performance with simpler experimentation.
If syntax independence is desired, we could use $\mathtt{simP}^{eq}(t_1, t_2)$, which returns $1$ if $t_1 = t_2$, and $0$ otherwise.

\begin{definition} \label{def:flatlit}
\upshape{
Let $\lambda \in (0,1)$, the \textbf{flat literal similarity measure} between two literals $(P, \vec{a})$, $(Q, \vec{b})$ is defined as: \\
$\simL^{\text{flat},\lambda}((P, \vec{a}), (Q, \vec{b})) = \frac{1}{2}(\simP(P, Q) + \similarity_{\text{para}}^{\text{}\lambda}(\vec{a}, \vec{b}))$
The parameters similarity \(\similarity_{\text{para}}^{\text{}\lambda}(\vec{a}, \vec{b})\) combines two complementary strategies:\\
$\similarity_{\text{para}}^{\text{}\lambda}(\vec{a}, \vec{b}) =  \lambda \times \simFlatPos(\vec{a}, \vec{b}) + (1 - \lambda) \times \simFlatUnord(\vec{a}, \vec{b})$
where $\lambda$ is a weighting parameter that controls the trade-off between order-sensitive and order-invariant elements.

\begin{itemize}
  \item \textit{Ordered similarity}: 
  Compares terms positionally, assuming meaningful arity and order: \\
      \phantom{---------} $\simFlatPos(\vec{a},\vec{b}) = \dfrac{\sum_{i=1}^{\min(|\vec{a}|, |\vec{b}|)}  \simP(a_i, b_i)}{\min(|\vec{a}|, |\vec{b}|)}$
  
  \item \textit{Best-match (unordered) similarity}: Matches each term to its most similar counterpart in the other vector: \\
  {
      \phantom{---} $\simFlatUnord(\vec{a},\vec{b}) = \frac{\sum\limits_{a_i \in \vec{a}} \simP(a_i, \best{a_i}{\vec{b}}) + \sum\limits_{b_j \in \vec{b}} \simP(b_j, \best{b_j}{\vec{a}})}{|\vec{a}| + |\vec{b}|}$
  }

where \(\best{x}{\vec{y}}\) is the best-matching element of the vector \(\vec{y}\) for a given term \(x\), based on a similarity measure $\simP$: \\
    \phantom{---------------} $\best{x}{\vec{y}} = \mathtt{arg max}_{y \in \vec{y}} \simP(x, y)$
\end{itemize}

\begin{example}[cont.]
    Let us compute the flat literal similarity between  $\texttt{Tease(dog,monkey)}$ and $\texttt{Tease(monkey,dog)}$ which can be represented \linebreak as $L_{1} = (Tease,\langle dog, monkey \rangle)$ and $L_{2} = (Tease,\langle monkey, dog \rangle)$ respectively.
    
    All scores are based on $\simP = \simsbert$, which returns $1$ for identical elements:   $\simsbert(\texttt{Tease}, \texttt{Tease}) = \simsbert(\texttt{monkey}, \texttt{monkey}) = \simsbert(\texttt{dog}, \texttt{dog}) = 1$, while 
    for $\simsbert(\texttt{monkey}, \texttt{dog})$ it returns $0.466$. Then, 
    $\simFlatPos(\tuple{\texttt{dog,monkey}}, \tuple{\texttt{monkey,dog}}) \simeq 0.466$; \\
    $\simFlatUnord(\tuple{\texttt{dog,monkey}}, \tuple{\texttt{monkey,dog}}) = 1$; \\
    $\simFlatPara{0.8}(\langle \texttt{dog,monkey} \rangle, \langle \texttt{monkey,dog} \rangle) \simeq \\ \phantom{----} 0.8 \times 0.466 + 0.2 \times 1 \simeq 0.573$. \\
    Hence, $\simL^{\text{flat},0.8}(L_1, L_2) \simeq \frac{1 + 0.573}{2} \simeq 0.786$.
\end{example}

}
\end{definition}

\textbf{Flat Clause-Level Similarity}.~~
We adapt the fuzzy Tversky measure proposed by~\cite{coletti2019fuzzy}, replacing fixed membership  with a similarity membership function.  
In particular, we use literal-level similarity to define membership when comparing clauses, enabling structured similarity computation.

\begin{definition}\label{def:membershipfunction}
\upshape{
Let $\setObj$ be a set of objects, $x \in \setObj$, $Y \subseteq \setObj$, $\agg$ an aggregation function 
and $\similarity\setObj$ a similarity measure on $\setObj$, the \textbf{membership function} of $x$ in $Y$, $\oplus^{\agg}_{\similarity\setObj} : \setObj \times 2^\setObj \rightarrow [0,1]$ is defined by: 
$\oplus^{\agg}_{\similarity \setObj}(x,Y) = \agg_{y \in Y} (\similarity \setObj(x,y))$.
}
\end{definition}

\begin{definition} \label{def:ftve}
\upshape{
Let two clauses $C_1$ and $C_2$, $\alpha,\beta \in [0,+\infty)$ and $\oplus^{\agg}_{\simL}$ a membership function on literals. The \textbf{fuzzy Tversky} similarity measure is defined as follows:\\
    \phantom{-----} $\text{Tve}^{\alpha,\beta,\oplus^{\agg}_{\simL}}(C_1,C_2) = \dfrac{A}{A + \alpha B + \beta C}, ~\text{where:}$\\
{\small
\phantom{-----} $A = \frac{1}{2} \left( \sum_{x \in C_1} \oplus^{\agg}_{\simL}(x,C_2) + \sum_{y \in C_2} \oplus^{\agg}_{\simL}(y,C_1) \right)$ \\
\phantom{-----} $B = \sum_{x \in C_1} \left( 1 - \oplus^{\agg}_{\simL}(x,C_2) \right)$ \\
\phantom{-----} $C = \sum_{y \in C_2} \left( 1 - \oplus^{\agg}_{\simL}(y,C_1) \right)$ 
}
}
\end{definition}

Classical similarity measures can be seen as instances of the Tversky measure \cite{Tversky77}, as \cite{Jaccard}, i.e. $\jacc$, obtained with $\alpha = \beta = 1$, \cite{Dice}, i.e. $\dice$, with $\alpha = \beta = 0.5$, \cite{Sorensen}, i.e. $\soren$ with $\alpha = \beta = 0.25$,  \cite{Anderberg}, i.e. $\ander$, with $\alpha = \beta = 0.125$, and \cite{Sneath}, i.e. $\sok$, with $\alpha = \beta = 2$. These classical measures can thus be extended to the fuzzy case.

In the following, we denote by $\simC^{\text{flat}}$ the family of flat clause similarity measures using the fuzzy Tversky formulation with $\simL = \simL^{\text{flat},\lambda}$.

\begin{definition} \label{def:lit-flat-similarity}
\upshape{
Let two clauses $C_1$ and $C_2$, $\alpha,\beta \in [0,+\infty)$ and $\oplus^{\agg}_{\simL}$ a membership function on literals.
Let $\simL = \simL^{\text{flat},\lambda}$ be a flat literal similarity measure. The \textbf{flat clause similarity measure} between $C_1$ and $C_2$ is defined as:\\
\phantom{---------} $\simC^{\text{flat}}(C_1,C_2) = \text{Tve}^{\alpha,\beta,\oplus^{\agg}_{\simL}}(C_1,C_2)$
}
\end{definition}

\noindent An instance of this family, using \cite{Dice}, is:\\
   \phantom{---------} $\simC^{\text{flat}}_{\dice}(C_1,C_2) = \text{Tve}^{\frac{1}{2},\frac{1}{2},\oplus^{\max}_{\simL^{\text{flat},0.8}}}(C_1,C_2)$

\begin{example}[cont.]
In this case, where each clause contains only one literal, the flat clause similarity $\simC^{\text{flat}}_{\dice}(C_1, C_2)$ corresponds to the flat literal similarity $\simL^{\text{flat}}_{\dice}(L_{1}, L_{2})$.
\end{example}

\subsection{Step 2: Weighted Clause-Level Similarity}
We now introduce weights to reflect the contextual importance of each predicate and term. 
We assume they are normalized to sum to 1, enabling similarity to be computed as a relative, interpretable aggregation, analogous to an attention distribution. 
In the following, we use $w_p$ to denote both predicate weights and term (parameter) weights.

\begin{definition}
\upshape{
Let $\lambda \in (0,1)$, the \textbf{weighted literal similarity measure} between two literals $(P, \vec{a})$, $(Q, \vec{b})$ is defined as: 
$\simL^{\text{weight},\lambda}((P, \vec{a}), (Q, \vec{b})) =$
\begin{center}
{\large
$\frac{w_p(P) w_p(Q) \times \simP(P, Q) ~ + ~ \similarity_{\text{para}}^{\text{}\lambda}(\vec{a}, \vec{b}) \times (\lambda w_\text{ord} + (1-\lambda)  w_\text{unord})}
{w_p(P) w_p(Q) ~+~ (\lambda w_\text{ord} + (1-\lambda)  w_\text{unord})}$
}
\end{center}

\noindent where 
$w_\text{ord} = \sum_{i=1}^{\min(|\vec{a}|, |\vec{b}|)} w_{p}(a_i) w_{p}(b_i)$ and \linebreak
$w_\text{unord} = \sum_{a_i \in \vec{a}} w_{p}(a_i)\, w_{p}(\best{a_i}{\vec{b}})
     +\!
     \sum_{b_j \in \vec{b}} w_{p}(b_j)\, w_{p}(\best{b_j}{\vec{a}})$.
}
\end{definition}

This formulation emphasizes salient elements by using the contextual importance weights of the predicates, i.e., \( w_p(P) w_p(Q) \), and those of the parameters, \( w_\text{ord} \) and \( w_\text{unord} \), which reflect the computation of \( \simFlatPos \) and \( \simFlatUnord \).

\begin{example}[cont.]
We have $w_p(\texttt{Tease}) = 0.1$, $w_p(\texttt{AtLoc}) = 0.1$, $w_p(\texttt{dog}) = 0.35$, $w_p(\texttt{monkey}) = 0.35$, and $w_p(\texttt{zoo}) = 0.1$, reflecting greater importance assigned to the information \texttt{dog} and \texttt{monkey}.\\
$\similarity_{\text{para}}^{\text{}0.8}(\tuple{\texttt{dog,monkey}}, \tuple{\texttt{monkey,dog}}) \simeq \\ 
\phantom{--.-.}0.8 \times 0.466 + 0.2 \times 1 \simeq 0.573$;\\
$w_\text{ord} \phantom{-} = (0.35 \times 0.35) + (0.35 \times 0.35) = 0.245$;  \\
$w_\text{unord} = 4 \times (0.35 \times 0.35) = 0.49$;\\
$\simL^{\text{weight},0.8}(L_{1}, L_{2}) \simeq \\ \phantom{--.-.} ${\large $\frac{0.01 \times 1 +  0.573 \times (0.8 \times 0.245+ 0.2 \times 0.49)}{0.01 + (0.8 \times 0.245+ 0.2 \times 0.49)}$} $\simeq 0.587$.
\end{example}

\textbf{Weighted Clause-Level Similarity.} ~
We now extend clause similarity by integrating contextual weights over literals.  
Using the same structure as Definition~\ref{def:lit-flat-similarity}, but replacing $\simL$ with $\simL^{\text{weight},\lambda}$, we obtain:\\
\phantom{----------} $\simC^{\text{weight}}(C_1,C_2) = \text{Tve}^{\alpha,\beta,\oplus^{\agg}_{\simL}}(C_1,C_2)$

A concrete instance with Dice weighting ($\alpha = \beta = \frac{1}{2}$) is:
    \phantom{-----} $\simC^{\text{weight}}_{\dice}(C_1,C_2) =\text{Tve}^{\frac{1}{2},\frac{1}{2},\oplus^{\max}_{\simL^{\text{weight},0.8}}}(C_1,C_2)$

\subsection{Step 3: Formulae-Level Similarity}
Let $w_c(C)$ be the contextual weight of clause $C$, with weights normalized over a set of formulae $\Phi$ (e.g., premises or claim), such that $\sum_{C \in \Phi} w_c(C) = 1$.

\begin{definition} \label{def:simSbm}
\upshape{
The \textbf{weighted average best match similarity measure} between two sets of clauses \( \Phi \), \( \Psi \) is defined as: \\
   \phantom{---} $\simS^{\bm}(\Phi, \Psi) = \frac{
\sum\limits_{(C_1, C_2) \in \bm(\Phi, \Psi)}
 w_g(C_1, C_2) \times \simC(C_1, C_2)
}{
\sum\limits_{(C_1, C_2) \in \bm(\Phi, \Psi)}
w_g(C_1, C_2)
}$

where $\bm$ relies on the flat clause similarity $\simC^{\text{flat}}$ to define the best matching clause pairs, i.e., $\bm(\Phi, \Psi) =$\\
{\small
\[
\left\{
\begin{aligned}
(C_1, C_2) \mid\, & \big(C_1 \in \Phi,\, C_2 = \mathtt{argmax}_{C' \in \Psi} \simC^{\text{flat}}(C_1, C')\big) \\
\text{or } &  \big(C_1 \in \Psi,\, C_2 = \mathtt{argmax}_{C' \in \Phi} \simC^{\text{flat}}(C_1, C')\big)
\end{aligned}
\right\}
\]
}

\noindent and where $w_g$ denotes the \emph{contribution weight} of a clause pair, computed using an aggregation function $\agg$:\\
    \phantom{--------------} $w_g(C_1, C_2) = \agg( w_{c}(C_1),w_{c}(C_2)).$
}
\end{definition}

In practice, $\simC$ is instantiated as the weighted clause similarity defined by a model, using a specific $\simP$ and literal similarity $\simL^{\text{weight}, \lambda}$.  
The flat version $\simC^{\text{flat}}$ uses the same $\simP$ and $\lambda$ with $\simL^{\text{flat}, \lambda}$ to ensure consistency.

\begin{example}[cont.] 
We have ${w_c(\texttt{AtLoc(dog, zoo)})}$ \linebreak ${ = w_c(\texttt{AtLoc(monkey, zoo)})} = 0.05$ and  \linebreak {\small ${ w_c(\texttt{Tease(dog, monkey)}) = w_c(\texttt{Tease(monkey, dog)})}$} \linebreak $= 0.9$.  
To compute $\simS^{\bm}_{\dice}(\text{CNF\_T1}, \text{CNF\_T2})$, we match each clause to its best counterpart (via $\simC^{\text{flat}}_{\dice}$), weight the similarity (via $\simC^{\text{weight}}_{\dice}$) by the average clause weight $w_\avg$, and sum.  
The \texttt{Tease} clauses contribute $2 \times 0.587 \times 0.9 \simeq 1.056$;  
the \texttt{AtLoc} clauses contribute $4 \times 1 \times 0.05 = 0.2$.  
Hence:
\[
\simS^{\bm}_{\dice}(\text{CNF\_T1}, \text{CNF\_T2}) = \frac{1.056 + 0.2}{2 \times 0.9 + 4 \times 0.05} \simeq 0.628.
\]
\end{example}

The contribution weight of a clause pair reflects its importance to the overall formulae similarity, independently of its similarity score. In Figure~\ref{fig:exTease}, these contributions are normalized as proportions (denoted $\|w_g\|$), summing to 1 across all best-matching clause pairs. Even perfect matches contribute little if their weight is low ($2 \times \|w_\avg\| = 0.05$), while moderately similar clauses like \texttt{Tease}, with high weight ($2 \times \|w_\avg\| = 0.9$), dominate the final score. The factor 2 reflects that each clause is matched in both directions. 

\begin{figure}[t]
\centering
\includegraphics[width=1\linewidth, height=0.5\linewidth]{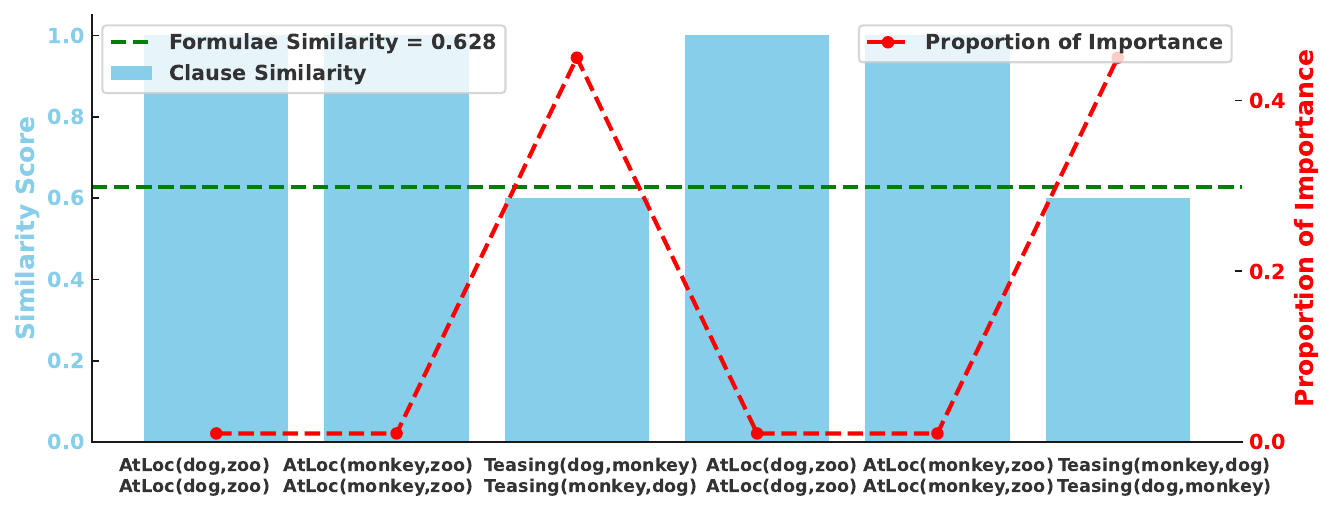}
\caption{
Histogram explaining the formulae similarity for [\textit{CNF\_T1}, \textit{CNF\_T2}] (green line), based on weighted clause similarities (blue) and their proportion of importance (red).}
\label{fig:exTease}
\end{figure}

We compare our similarity model to state-of-the-art baselines on the previously introduced pair T1 vs. T2, which involves a subtle agent/patient role reversal. 
SBERT produces a high similarity score of 0.985 and  provides no explanation. 
S3BERT~\cite{opitz2022sbert} returns 0.990 and provides interpretability via Abstract Meaning Representation~\cite{banarescu2013abstract} similarity, e.g., Semantic Role Labeling (SRL) similarity = 0.96. However, such explanations often lack clarity for non-experts. 
ChatGPT-4o, prompted in zero-shot, returns a lower score of 0.400 and the following textual justification:
\emph{``The two sentences share the same entities and location, but the reversed agent and patient roles change the core meaning, reducing their similarity."}
While insightful, this explanation lacks structure and does not quantify the relative impact of each semantic component. 

Our approach yields a similarity score of 0.628 (computed in approximately 0.3 seconds on an M2 Pro chip with 32 GB of RAM), closely aligned with the average human score of 0.630 \cite{cer2017semeval}, thanks to appropriate contextual weighting. Unlike prior methods, it provides a symbolic decomposition with contextual weights, enabling transparent and modular explanations. For instance, it detects perfect matches on location clauses such as \texttt{AtLoc(dog,zoo)} and \texttt{AtLoc(monkey,zoo)} (each weighted at 5\%), while assigning lower similarity (0.587) to \texttt{Tease(dog,monkey)} vs. \texttt{Tease(monkey,dog)}, which is given a dominant weight (90\%) due to its contextual importance. The final score reflects this imbalance.

We stress that these results are based on a limited number of examples and do not constitute a formal evaluation of predictive performance. However, they illustrate the promise of our approach as a novel, interpretable framework for similarity assessment, grounded in logical structure and enriched with contextual weighting, offering a compelling alternative to opaque embedding-based or purely textual methods.

\subsection{Step 4: Arguments Similarity}
Let us consider three natural language arguments:\\
    $\mathbf{(A_1)}$ $\langle\{$A dog is teasing a monkey at the zoo, and teasing usually indicates dominance or playful behavior$\}$, so the dog is likely exhibiting dominant or playful behavior.$\rangle$\\
    $\mathbf{(A_2)}$ $\langle\{$A dog and a monkey are teasing each other at the zoo. Mutual teasing suggests a playful relationship.$\}$ So the dog and monkey have a playful interaction.$\rangle$\\
    $\mathbf{(A_3)}$ $\langle\{$A dog and a monkey are teasing each other at the zoo. Mutual teasing reflects a social conflict.$\}$ So the dog and the monkey may be engaged in a social conflict.$\rangle$

Using $\simS^{\bm}_{\dice}$, we compute the similarity between premises and claims, and aggregate them into overall argument similarity as in Definition~\ref{def:smArguments}.  
For $A_1$ and $A_2$, premise similarity is $0.795$ and claim similarity is $0.757$, yielding:
$\mathtt{simArg}(A_1, A_2) = 0.795\,\eta + 0.757\,(1-\eta)$. 
For $A_2$ and $A_3$, we have $0.913$ and $0.653$, giving:
$\mathtt{simArg}(A_2, A_3) = 0.913\,\eta + 0.653\,(1-\eta)$.
The crossover occurs at $\eta = 0.468$: below this threshold, $A_2$ is closer to $A_3$; above it, $A_2$ is closer to $A_1$.
This example illustrates the challenge of comparing arguments and how our method adapts flexibly to different contexts.  
For instance, in debate scenarios with multiple attacks on the same argument, greater importance may be given to the similarity between the attackers' claims, as they often highlight the main disagreement.
While in enthymeme decoding, where an incomplete argument is compared to its reconstruction, premises and claim may carry equal weight.

\section{Axiomatic Evaluation}\label{section:axiomatic-evaluation}
In this section, we formally analyze how our similarity models align with the proposed principles.  
We introduce two families of argument similarity measures and evaluate their satisfaction.  
We also study a key case: when the similarity score equals 1, indicating perfect similarity.  

The next theorems specify conditions on each component of a similarity model $\bM = \tuple{\simP, \simL, \simC, \simS}$ 
to \linebreak ensure that their combination satisfies the principles.

\begin{theorem}\label{theo:wbMax}
\upshape{
$\similarity\mathtt{Arg}^{\bM}_{\eta}$  satisfies \textbf{Maximality} if:\\
    1.\phantom{-}$\simP(t,t) = 1$.\\
    2.\phantom{-}$\simL((P, \vec{a}), (Q, \vec{b})) = 1$ if $\simP(P,Q) = 1$, $\card{\vec{a}} = \card{\vec{b}}$, \linebreak
    \phantom{2.-}and $\forall a_i \in \vec{a}, \forall b_i \in \vec{b}$, $\simP(a_i,b_i) = 1$.\\
    3.\phantom{-}$\simC(C_1,C_2) = 1$ if $\forall l_1 \in C_1,\exists l_2 \in C_2$ s.t. $\simL(l_1,l_2)$ \linebreak
    \phantom{3.-}$ = 1$ and $\forall l_2 \in C_2,\exists l_1 \in C_1$ s.t. $\simL(l_2,l_1) = 1$.\\
    4.\phantom{-}$\simS(\Phi_1,\Phi_2) = 1$ if $\forall C_1 \in \Phi_1,\exists C_2 \in \Phi_2$ s.t. \linebreak
    \phantom{4.-}$\simC(C_1, C_2) = 1$ and $\forall C_2 \in \Phi_2,\exists C_1 \in \Phi_1$ s.t. \linebreak
    \phantom{4.-}$\simC(C_2, C_1) = 1$.
}
\end{theorem}

\begin{theorem}\label{theo:wbSym}
\upshape{
$\similarity\mathtt{Arg}^{\bM}_{\eta}$  satisfies \textbf{Symmetry} if $\simP, \simL, \simC,$ and $\simS$ are symmetric.
}
\end{theorem}

\begin{theorem}\label{theo:wbSub}
\upshape{
$\similarity\mathtt{Arg}^{\bM}_{\eta}$  satisfies \textbf{Substitution} if $\simS(\Phi_1, \Phi_2) = 1$ then $\simS(\Phi_1,\Phi_3) = \simS(\Phi_2,\Phi_3)$.
}
\end{theorem}

\begin{theorem}\label{theo:wbSI}
\upshape{
$\similarity\mathtt{Arg}^{\bM}_{\eta}$  satisfies \textbf{Syntax Independence} if $\simP(t_1,t_1) = 1$ and $\simP(t_1,t_2) = 0$.
}
\end{theorem}

\begin{theorem}\label{theo:wbMin}
\upshape{
$\similarity\mathtt{Arg}^{\bM}_{\eta}$  satisfies \textbf{Minimality} if:\\
    1.\phantom{-}$\simL((P, \vec{a}), (Q, \vec{b})) = 0$ if $\simP(P,Q) = 0$ and \linebreak
    \phantom{1.-}$\forall a_i \in \vec{a},\forall b_j \in \vec{b}$, $\simP(a_i,b_j) = 0$.\\
    2.\phantom{-}{\midsize $\simC(C_1,C_2) = 0$ if $\forall l_1 \in C_1,\forall l_2 \in C_2$, $\simL(l_1,l_2) = 0$.}\\
    3.\phantom{-}$\simS(\Phi_1,\Phi_2) = 0$ if $\forall C_1 \in \Phi_1,\forall C_2 \in \Phi_2$, \linebreak
    \phantom{3.-}$\simC(C_1, C_2) = 0$.
}
\end{theorem}

\begin{theorem}\label{theo:wbNZ}
\upshape{
$\similarity\mathtt{Arg}^{\bM}_{\eta}$  satisfies \textbf{Non-Zero} if:\\
    1.\phantom{-}{\midsize$\simL((P, \vec{a}), (Q, \vec{b})) > 0$, if $\simP(P,Q) > 0$ or, there exists \linebreak
    \phantom{1.-}a position $i$ s.t. $a_i \in \vec{a}$, $b_i \in \vec{b}$,  $\simP(a_i,b_i) > 0$.}\\
    2.\phantom{-}$\simC(C_1, C_2) > 0$ if $\exists l_1 \in C_1$, $l_2 \in C_2$ s.t. \linebreak
    \phantom{2.-}$\simL(l_1, l_2) > 0$. \\
    3.\phantom{-}$\simS(\Phi_1,\Phi_2) > 0$ if $\exists C_1 \in \Phi_1$ and $\exists C_2 \in \Phi_2$ s.t. \linebreak
    \phantom{3.-}$\simC(C_1, C_2) > 0$. 
}
\end{theorem}

\begin{theorem}\label{theo:wbMon}
\upshape{
$\similarity\mathtt{Arg}^{\bM}_{\eta}$  satisfies the principles of \linebreak \textbf{S-Monotony$^0$}, \textbf{S-Monotony$^1$}, \textbf{C-Monotony$^0$}, and \linebreak \textbf{C-Monotony$^1$} if 
     $\simS(\Phi \cup \{C_0\},\Psi) \leq \simS(\Phi,\Psi) \leq \simS(\Phi \cup \{C_1\},\Psi)$ when $\max_{C \in \Psi} \simC(C,C_0) = 0$ and $\max_{C \in \Psi}  \simC(C,C_1) = 1$.
}
\end{theorem}

\begin{theorem}\label{theo:wbRein}
\upshape{
$\similarity\mathtt{Arg}^{\bM}_{\eta}$  satisfies the principles of \textbf{S-Reinforcement$^\geq$}, \textbf{S-Reinforcement$^>$}, \linebreak \textbf{C-Reinforcement$^\geq$}, and \textbf{C-Reinforcement$^>$} if 
    $\forall C \in \Psi,\ \simC(C_A,C) \geq \simC(C_B,C)$ then $\simS(\Phi \cup \{C_A\}, \Psi) \geq \simS(\Phi \cup \{C_B\}, \Psi)$ and 
    if also $\exists C \in \Psi \text{ s.t. } \simC(C_A,C) > \simC(C_B,C)$ then $\simS(\Phi \cup \{C_A\}, \Psi) >\simS(\Phi \cup \{C_B\}, \Psi)$.
}
\end{theorem}

In the following, we study two families of instantiations of our similarity model: the first one is syntax-dependent $\bM^{x,w_g}_{sb,\lambda} = \tuple{\simsbert, \simL, \simC, \simS}$, and the second one is not $\bM^{x,w_g}_{eq,\lambda} = \tuple{\simP^{eq}, \simL, \simC, \simS}$, where:\\
    $\bullet$ $\simL = \simL^{\text{weight},\lambda}$ with $\lambda \in (0,1)$, and $\simP$ is set to $\simsbert$ for $\bM^{x,w_g}_{sb,\lambda}$ and to $\simP^{eq}$ for $\bM^{x,w_g}_{eq,\lambda}$; \\
    $\bullet$ $\simC = \text{Tve}^{x,\oplus^{\max}_{\simL}}$ with $x \in \{\jacc, \dice, \soren, \ander, \sok\}$; \\
    $\bullet$ $\simS = \simS^{\bm}$, where 
    $\simC^{\text{flat}}$ uses $\simP = \simsbert$ for $\bM^{x,w_g}_{sb,\lambda}$ and $\simP = \simP^{eq}$ for $\bM^{x,w_g}_{eq,\lambda}$ within $\simL^{\text{flat}, \lambda}$, and $w_g$ is average ($w_\mathtt{avg}$) or product ($w_\Pi$).

\begin{theorem} \label{theo:satisfaction-principles-measures}
\upshape{
    For any $\eta \in ~(0,1)$, $x \in \{\jacc, \dice, \soren, \ander, \sok\}$, $w_g \in \{w_\mathtt{avg},w_\Pi\}$ and $\lambda \in (0,1)$, $\similarity\mathtt{Arg}^{\bM^{x,w_g}_{sb,\lambda}}_\eta$ satisfies all principles except Syntax Independence and Non-Zero, while $\similarity\mathtt{Arg}^{\bM^{x,w_g}_{eq,\lambda}}_\eta$ satisfies all except Non-Zero, i.e., 13 out of the 14 principles.
}
\end{theorem}

While Non-Zero may seem intuitive by focusing on local similarities, human ratings often violate it: in STSb, \textit{``a man is smoking''} vs. \textit{``a baby is sucking''} has a global similarity of  0, despite some overlap on \textit{man} and \textit{baby}. 
By favoring dissimilar over similar elements, our models also violate it.

\begin{proposition}\label{prop:compatibility}
\upshape{
All principles are compatible.    
}
\end{proposition}

In the following theorem, we show the connections between argument equivalence and total similarity.

\begin{theorem}\label{theo:=1}
\upshape{
    Let two CNF arguments $A, B \in \ArgF$, for any $\eta \in ~ (0,1)$: 
        $\bullet$ if $A \approx B$ then $\similarity\mathtt{Arg}^{\bM^{x,w_g}_{sb,\lambda}}_\eta(A,B) = 1$.\\
        \phantom{-------------.-.---} $\bullet$  $\similarity\mathtt{Arg}^{\bM^{x,w_g}_{eq,\lambda}}_\eta(A,B) = 1$ iff $A \approx B$.
    }
\end{theorem}

In practice, $\simsbert$ returns a score of~1 only for syntactically identical texts, and semantically equivalent concepts like \texttt{notEven} and \texttt{Odd} do not reach full similarity. 
However, SBERT provides no formal guarantees on its embedding space, which is why we cannot establish the characterization result in the first bullet point.

\section{Conclusion and Future Work}
We introduce a new framework for measuring similarity between FOL arguments, based on a model with four levels: predicates/terms, literals, clauses, and formulae.  
Our main contributions are:   
1. An extended axiomatic foundation providing theoretical guarantees for similarity computations;  
2. A structured, parametric model that leverages the internal FOL structure to enhance the understanding of argument similarity;  
3. Two model families, one syntax-sensitive (via language models) and one syntax-independent, both integrating contextual weights that reflect the relative importance perceived by humans; and 
4. Formal constraints that guide the design of similarity measures, ensuring they satisfy desirable properties. 
This work presents the first hybrid framework for argument similarity, combining FOL-CNF structure, language model, and contextual weighting to jointly capture both structural and semantic aspects.  
The resulting model is theoretically grounded and shows strong potential for applications in natural language argumentation.

Our framework opens several research directions. On the theoretical side, we plan to explore non-symmetric similarity measures, where the extended fuzzy Tversky measure offers a natural foundation for this. We also aim to define layer-specific principles to guide the definition of similarity functions at each layer, extending our axiomatic work and enhancing modular interpretability. 
On the empirical side, we will evaluate the framework on textual similarity benchmarks and analyze the impact of different parameter choices. Finally, we will investigate how contextual weight calibration affects score sensitivity and interpretability.

\section*{Acknowledgment}
All authors are supported by the French National Research Agency (ANR grant AGGREEY ANR-22-CE23-0005).
The work by Victor David was supported by the French government, managed by the Agence Nationale de la Recherche under the Plan d’Investissement France 2030, as part of the Initiative d’Excellence d’Université Côte d’Azur under the reference ANR-15-IDEX-01.
The work by Jean-Guy Mailly is supported by French National Research Agency (ANR grant AIDAL ANR-22-CPJ1-0061-01).

\bibliography{biblio}

@inproceedings{david2025similarity,
  title={Similarity Measures for First-Order Logical Arguments},
  author={David, Victor and Delobelle, J{\'e}r{\^o}me and Mailly, Jean-Guy},
  booktitle={NMR 2025-23rd International Workshop on Nonmonotonic Reasoning},
  year={2025}
}

@inproceedings{david2025logic,
  title={A logic-based framework for decoding enthymemes in argument maps involving implicitness in premises and claims},
  author={David, Victor and Hunter, Anthony},
  booktitle={IJCAI 2025-Thirty-Fourth International Joint Conference on Artificial Intelligence},
  pages={4445--4453},
  year={2025},
  organization={International Joint Conferences on Artificial Intelligence Organization}
}

@inproceedings{han2022folio,
  author       = {Simeng Han and
                  Hailey Schoelkopf and
                  Yilun Zhao and
                  Zhenting Qi and
                  Martin Riddell and
                  Wenfei Zhou and
                  James Coady and
                  David Peng and
                  Yujie Qiao and
                  Luke Benson and
                  Lucy Sun and
                  Alexander Wardle{-}Solano and
                  Hannah Szab{\'{o}} and
                  Ekaterina Zubova and
                  Matthew Burtell and
                  Jonathan Fan and
                  Yixin Liu and
                  Brian Wong and
                  Malcolm Sailor and
                  Ansong Ni and
                  Linyong Nan and
                  Jungo Kasai and
                  Tao Yu and
                  Rui Zhang and
                  Alexander R. Fabbri and
                  Wojciech Kryscinski and
                  Semih Yavuz and
                  Ye Liu and
                  Xi Victoria Lin and
                  Shafiq Joty and
                  Yingbo Zhou and
                  Caiming Xiong and
                  Rex Ying and
                  Arman Cohan and
                  Dragomir Radev},
  title        = {{FOLIO:} Natural Language Reasoning with First-Order Logic},
  booktitle    = {Proceedings of the 2024 Conference on Empirical Methods in Natural Language Processing, {EMNLP}},
  pages        = {22017--22031},
  publisher    = {Association for Computational Linguistics},
  year         = {2024}
}

@inproceedings{lu2022parsing,
  title={Parsing natural language into propositional and first-order logic with dual reinforcement learning},
  author={Lu, Xuantao and Liu, Jingping and Gu, Zhouhong and Tong, Hanwen and Xie, Chenhao and Huang, Junyang and Xiao, Yanghua and Wang, Wenguang},
  booktitle={Proceedings of the 29th International Conference on Computational Linguistics},
  pages={5419--5431},
  year={2022}
}

@inproceedings{ryu2024divide,
  author       = {Hyun Ryu and
                  Gyeongman Kim and
                  Hyemin S. Lee and
                  Eunho Yang},
  title        = {Divide and Translate: Compositional First-Order Logic Translation and Verification for Complex Logical Reasoning},
  booktitle    = {The Thirteenth International Conference on Learning Representations, {ICLR}},
  publisher    = {OpenReview.net},
  year         = {2025}
}

@inproceedings{yang2023harnessing,
  author       = {Yuan Yang and
                  Siheng Xiong and
                  Ali Payani and
                  Ehsan Shareghi and
                  Faramarz Fekri},
  title        = {Harnessing the Power of Large Language Models for Natural Language to First-Order Logic Translation},
  booktitle    = {Proc. of the 62nd Annual Meeting of the Association for Computational Linguistics (Volume 1: Long Papers), {ACL}},
  pages        = {6942--6959},
  publisher    = {Association for Computational Linguistics},
  year         = {2024}
}

@inproceedings{david2023similarity,
  title={Similarity Measures between Order-Sorted Logical Arguments},
  author={David, Victor and Delobelle, J{\'e}r{\^o}me and Mailly, Jean-Guy},
  booktitle={Journ{\'e}es d'Intelligence Artificielle Fondamentale},
  year={2023}
}

@article{huang2024cosent,
  title={Cosent: Consistent sentence embedding via similarity ranking},
  author={Huang, Xiang and Peng, Hao and Zou, Dongcheng and Liu, Zhiwei and Li, Jianxin and Liu, Kay and Wu, Jia and Su, Jianlin and Yu, Philip S},
  journal={IEEE/ACM Transactions on Audio, Speech, and Language Processing},
  year={2024},
  publisher={IEEE}
}

@inproceedings{banarescu2013abstract,
  title={Abstract meaning representation for sembanking},
  author={Banarescu, Laura and Bonial, Claire and Cai, Shu and Georgescu, Madalina and Griffitt, Kira and Hermjakob, Ulf and Knight, Kevin and Koehn, Philipp and Palmer, Martha and Schneider, Nathan},
  booktitle={Proceedings of the 7th linguistic annotation workshop and interoperability with discourse},
  pages={178--186},
  year={2013}
}

@inproceedings{opitz2022sbert,
  author       = {Juri Opitz and
                  Anette Frank},
  title        = {{SBERT} studies Meaning Representations: Decomposing Sentence Embeddings into Explainable Semantic Features},
  booktitle    = {Proc. of the 2nd Conference of the Asia-Pacific Chapter of the Association for Computational Linguistics and the 12th International
                  Joint Conference on Natural Language Processing, {AACL/IJCNLP}},
  pages        = {625--638},
  publisher    = {Association for Computational Linguistics},
  year         = {2022}
}

@article{cer2017semeval,
  title={Semeval-2017 task 1: Semantic textual similarity-multilingual and cross-lingual focused evaluation},
  author={Cer, Daniel and Diab, Mona and Agirre, Eneko and Lopez-Gazpio, Inigo and Specia, Lucia},
  journal={arXiv preprint arXiv:1708.00055},
  year={2017}
}

@phdthesis{David21,
  author       = {Victor David},
  title        = {Dealing with Similarity in Argumentation. (Traitement de la Similarit{\'{e}}
                  en Argumentation)},
  school       = {Paul Sabatier University, Toulouse, France},
  year         = {2021},
  url          = {https://tel.archives-ouvertes.fr/tel-03578375},
  bibsource    = {dblp computer science bibliography, https://dblp.org}
}

@inproceedings{lee2025entailment,
  author       = {Jinu Lee and
                  Qi Liu and
                  Runzhi Ma and
                  Vincent Han and
                  Ziqi Wang and
                  Heng Ji and
                  Julia Hockenmaier},
  title        = {Entailment-Preserving First-order Logic Representations in Natural Language Entailment},
  booktitle    = {Proceedings of the 63rd Annual Meeting of the Association for Computational Linguistics (Volume 1: Long Papers), {ACL}},
  pages        = {5729--5742},
  publisher    = {Association for Computational Linguistics},
  year         = {2025}
}

@article{lalwani2024nl2fol,
  title={NL2FOL: translating natural language to first-order logic for logical fallacy detection},
  author={Lalwani, Abhinav and Chopra, Lovish and Hahn, Christopher and Trippel, Caroline and Jin, Zhijing and Sachan, Mrinmaya},
  journal={arXiv preprint arXiv:2405.02318},
  year={2024}
}

@inproceedings{coletti2019fuzzy,
  title={Fuzzy similarity measures and measurement theory},
  author={Coletti, Giulianella and Bouchon-Meunier, Bernadette},
  booktitle={2019 IEEE international conference on fuzzy systems (FUZZ-IEEE)},
  pages={1--7},
  year={2019},
  organization={IEEE}
}

@article{ben2024understanding,
  title={Understanding Enthymemes in Argument Maps: Bridging Argument Mining and Logic-based Argumentation},
  author={Ben-Naim, Jonathan and David, Victor and Hunter, Anthony},
  journal={arXiv preprint arXiv:2408.08648},
  year={2024}
}

@article{ben2024axiomatic,
  title={An Axiomatic Study of the Evaluation of Enthymeme Decoding in Weighted Structured Argumentation},
  author={Ben-Naim, Jonathan and David, Victor and Hunter, Anthony},
  journal={arXiv preprint arXiv:2411.04555},
  year={2024}
}

@article{irwin2022forecasting,
  title={Forecasting argumentation frameworks},
  author={Irwin, Benjamin and Rago, Antonio and Toni, Francesca},
  journal={arXiv preprint arXiv:2205.11590},
  year={2022}
}

@inproceedings{guo2023argumentative,
  title={Argumentative Explanation for Deep Learning: A Survey},
  author={Guo, Yihang and Yu, Tianyuan and Bai, Liang and Tang, Jun and Ruan, Yirun and Zhou, Yun},
  booktitle={2023 IEEE International Conference on Unmanned Systems (ICUS)},
  pages={1738--1743},
  year={2023},
  organization={IEEE}
}

@inproceedings{vcyras2021argumentative,
  author       = {Kristijonas Cyras and
                  Antonio Rago and
                  Emanuele Albini and
                  Pietro Baroni and
                  Francesca Toni},
  title        = {Argumentative {XAI:} {A} Survey},
  booktitle    = {Proceedings of the Thirtieth International Joint Conference on Artificial Intelligence, {IJCAI}},
  pages        = {4392--4399},
  publisher    = {ijcai.org},
  year         = {2021}
}

@inproceedings{reimers2019sentence,
  author       = {Nils Reimers and
                  Iryna Gurevych},
  title        = {Sentence-BERT: Sentence Embeddings using Siamese BERT-Networks},
  booktitle    = {Proc. of the Conference on Empirical Methods in Natural Language Processing and the 9th International Joint Conference on
                  Natural Language Processing, {EMNLP-IJCNLP}},
  pages        = {3980--3990},
  publisher    = {Association for Computational Linguistics},
  year         = {2019}
}

@inproceedings{amgoud2021similarity,
  title={Similarity Measures Based on Compiled Arguments},
  author={L. Amgoud and V. David},
  booktitle={ECSQARU'21},
  pages={32--44},
  year={2021},
}

@inproceedings{AmgoudD21,
  TITLE = {{A General Setting for Gradual Semantics Dealing with Similarity}},
  AUTHOR = {L. Amgoud and V. David},
  BOOKTITLE = {{AAAI}'21},
  YEAR = {2021},
}

@inproceedings{AmgoudD20,
  author    = {L. Amgoud and
               V. David},
  title     = {An Adjustment Function for Dealing with Similarities},
  booktitle = {{COMMA}'20},
  pages     = {79--90},
  year      = {2020},
}

@inproceedings{AmgoudDD19,
  author    = {L. Amgoud and
               V. David and
               D. Doder},
  title     = {Similarity Measures Between Arguments Revisited},
  booktitle = {ECSQARU'19},
  pages     = {3--13},
  year      = {2019},
}

@inproceedings{besnard2005practical,
  title={Practical first-order argumentation},
  author={P. Besnard and A. Hunter},
  year={2005},
  booktitle = {AAAI'05},
  pages     = {590--595},
}

@article{ZhongFLT19,
  author    = {Q. Zhong and
               X. Fan and
               X. Luo and
               F. Toni},
  title     = {An explainable multi-attribute decision model based on argumentation},
  journal   = {Expert Sys. and Appl.},
  volume    = {117},
  pages     = {42--61},
  year      = {2019}
}

@inproceedings{AmgoudD18,
  author    = {L. Amgoud and V. David},
  title     = {Measuring Similarity between Logical Arguments},
  booktitle = {{KR}'18},
  pages     = {98--107},
  year      = {2018}
}

@article{jaccard,
	author    = {P. Jaccard},
	title     = {Nouvelles recherches sur la distributions florale},
	journal   = {Bulletin de la societe Vaudoise des sciences naturelles},
	volume    = {37},
	pages     = {223--270},
	year      = {1901}
}

@article{Dice,
  title={Measures of the amount of ecologic association between species},
  author={L. Dice},
  journal={Ecology},
  volume={26},
  number={3},
  pages={297--302},
  year={1945},
  publisher={Wiley Online Library}
}

@article{Sorensen,
  title={A method of establishing groups of equal amplitude in plant sociology based on similarity of species and its application to analyses of the vegetation on Danish commons},
  author={T. S{\o}rensen},
  journal={Biol. Skr.},
  volume={5},
  pages={1--34},
  year={1948}
}

@misc{Anderberg,
  title={Cluster analysis for applications. Monographs and textbooks on probability and mathematical statistics},
  author={M. Anderberg},
  year={1973},
  publisher={Academic Press, Inc., New York}
}

@book{Sneath,
  title={Numerical taxonomy. The principles and practice of numerical classification.},
  author={P. Sneath and R. Sokal},
  year={1973}
}

@inproceedings{Jantke94,
  author = {K. P. Jantke},
  title = {Nonstandard Concepts of Similarity in Case-Based Reasoning},
  booktitle = {Information Systems in Data Analysis: Prospects -- Foundations -- Applications},
  pages = {28--43},
  year = {1994},
}

@article{Tversky77,
 author = {A. Tversky},
 title = {Features of Similarity},
 journal = {Psychological Review},
 volume = {84},
 number = {4},
 pages = {327--352},
 year = {1977},
}

@inproceedings{LeturcB23,
  author       = {Christopher Leturc and
                  Flavien Balbo},
  title        = {{ADP:} An Argumentation-based Decision Process Framework Applied to the Modal Shift Problem},
  booktitle    = {Arg. \& App.@KR'23},
  pages        = {65--77},
  publisher    = {CEUR-WS.org},
  year         = {2023}
}

@inproceedings{Gorur0T23,
  author       = {Deniz Gorur and
                  Antonio Rago and
                  Francesca Toni},
  title        = {ArguCast: {A} System for Online Multi-Forecasting with Gradual Argumentation},
  booktitle    = {Arg. \& App.@KR'23},
  pages        = {40--51},
  publisher    = {CEUR-WS.org},
  year         = {2023}
}


\newpage
{\color{white}.}
\newpage

{\huge \textbf{Supplementary Material} }\\
\vspace{0.5cm}

{\Large \noindent \textbf{Similarity Between Arguments.}\\}

In this section, we illustrate how our similarity framework can be applied to structured arguments, where each argument is a pair consisting of \emph{premises} and a \emph{claim}. The similarity between two arguments is computed by comparing their premises and claims independently, using our multi-level logic-based method.
We then combine the two similarity scores using a weighted aggregation reflecting the importance of the premises versus the claim.

\paragraph{Case Study: Animals Tease Scenarios}

We evaluate the similarity between three arguments ($A_1$, $A_2$ and $A_3$) derived from natural language scenarios involving animals teasing each other at the zoo. 
Crucially, the weights assigned to an argument’s components may vary depending on the specific comparison being made. In other words, the importance of a given predicate or constant is not absolute, but contextual and relational. This is exemplified by argument $A_{2}$, where the most salient features shift depending on whether it is compared to $A_{1}$ or $A_{3}$. For instance, when comparing $A_{2}$ with $A_{1}$, the clause \texttt{Tease(dog, monkey)} may be less impactful than \texttt{Tease(monkey, dog)}, since the latter is absent from $A_{1}$. In such cases, we may choose to emphasize differences over shared information. However, when both clauses are also present in $A_{3}$, their relative importance can be downweighted and balanced more equally. This illustrates the need for a dynamic weighting scheme that adapts to the comparison context, rather than relying on static, global importance scores.

\paragraph{Argument 1: Dominance or Playful Behavior}
\begin{center} \it
    A dog is teasing a monkey at the zoo, and teasing usually indicates dominance or playful behavior, so the dog is likely exhibiting dominant or playful behavior.
\end{center}
For this argument, the weights assigned to constants, predicates and clauses was computed according to $A_2$.\\

\noindent \textbf{Premises}: \textit{``A dog is teasing a monkey at the zoo, and teasing usually indicates dominance or playful behavior''} can be formalized $\supp(A_1) = \{\alpha_{1},\alpha_{2},\alpha_{3},\alpha_{4}\}$ with:
\begin{itemize}
    \item $\alpha_{1} = \texttt{AtLocation(dog, zoo)}$
    \item $\alpha_{2} = \texttt{AtLocation(monkey, zoo)}$
    \item $\alpha_{3} =$ \texttt{$\neg$Tease(x, y)} $\lor$ \texttt{Dominant(x)} $\lor$ \texttt{Playful(x)}
    \item $\alpha_{4} = \texttt{Tease(dog, monkey)}$
\end{itemize}
The weights are distributed as follows:
    \begin{itemize}
        \item $w_{c}(\alpha_{1}) = w_{c}(\alpha_{2}) = 0.1$;
        \item $w_{c}(\alpha_{3}) = w_{c}(\alpha_{4}) = 0.4$;
        \item $w_{p}$(\texttt{AtLocation}) = 0.05, 
        $w_{p}$(\texttt{Tease}) = 0.05,\\
        $w_{p}$($\neg$\texttt{Tease}) = 0.05,
        $w_{p}$(\texttt{Dominant}) = 0.2,\\
        $w_{p}$(\texttt{Playful}) = 0.2; 
        \item  
        $w_{p}$(\texttt{dog}) = 0.2,
        $w_{p}$(\texttt{monkey}) = 0.2,
        $w_{p}$(\texttt{zoo}) = 0.05,
        $w_{p}$(\texttt{x}) = 0, 
        $w_{p}$(\texttt{y}) = 0.
    \end{itemize}
\textbf{Claim}: \textit{``The dog is likely exhibiting dominant or playful behavior''} can be formalized as $\conc(A_1) = \{\beta_1\}$ with:
    \begin{itemize}
        \item $\beta_1$ = \texttt{Dominant(dog) $\lor$ Playful(dog)}
    \end{itemize}
The weights are distributed as follows:
\begin{itemize}
    \item $w_{c}(\beta_1) = 1$; 
    \item $w_{p}$(\texttt{Dominant}) = 0.35 and $w_{p}$(\texttt{Playful}) = 0.35;
    \item $w_{p}$(\texttt{dog}) = 0.3.
\end{itemize}

\paragraph{Argument 2: Mutual Playful Interaction}
\begin{center} \it
    A dog is teasing a monkey at the zoo, and the monkey is teasing the dog at the zoo. When teasing is mutual, it suggests a playful relationship. So the dog and monkey have a playful interaction.
\end{center}
\textbf{Premises}: \textit{``A dog is teasing a monkey at the zoo, and the monkey is teasing the dog at the zoo. When teasing is mutual, it suggests a playful relationship.''} can be formalized $\supp(A_2) = \{\alpha_{5},\alpha_{6},\alpha_{7},\alpha_{8},\alpha_{9}\}$ with:
\begin{itemize}
    \item $\alpha_{5} = $ \texttt{AtLocation(dog, zoo)}
    \item $\alpha_{6} = $ \texttt{AtLocation(monkey, zoo)}
    \item $\alpha_{7} = $ $\neg$ \texttt{Tease(x, y)}  $\lor$ $\neg$ \texttt{Tease(y, x)} $\lor$ \texttt{PlayfulInteraction(x, y)}
    \item $\alpha_{8} = $ \texttt{Tease(dog, monkey)}
    \item $\alpha_{9} = $ \texttt{Tease(monkey, dog)}
\end{itemize}

\noindent The weights according to $A_1$ are distributed as follows:
\begin{itemize}
    \item $w_{c}(\alpha_{5}) = w_{c}(\alpha_{6}) = 0.05$;
    \item $w_{c}(\alpha_{7}) = w_{c}(\alpha_{9}) = 0.4$;
    \item $w_{c}(\alpha_{8}) = 0.1$;
    \item $w_{p}$(\texttt{AtLocation}) = 0.05, $w_{p}$(\texttt{Tease}) = 0.05, $w_{p}$($\neg$\texttt{Tease}) = 0.05, $w_{p}$(\texttt{PlayfulInteraction}) = 0.4;
    \item $w_{p}$(\texttt{dog}) = 0.2,
        $w_{p}$(\texttt{monkey}) = 0.2,
        $w_{p}$(\texttt{zoo}) = 0.05,
        $w_{p}$(\texttt{x}) = 0, 
        $w_{p}$(\texttt{y}) = 0
\end{itemize}

\noindent The weights according to $A_3$ are distributed as follows:
\begin{itemize}
    \item $w_{c}(\alpha_{5}) = w_{c}(\alpha_{6})= 0.05$;
    \item $w_{c}(\alpha_{7}) = w_{c}(\alpha_{8}) = w_{c}(\alpha_{9}) = 0.3$;
    \item $w_{p}$(\texttt{AtLocation}) = 0.05, $w_{p}$(\texttt{Tease}) = 0.05, $w_{p}$($\neg$\texttt{Tease}) = 0.05, $w_{p}$(\texttt{PlayfulInteraction}) = 0.4;
    \item $w_{p}$(\texttt{dog}) = 0.2,
        $w_{p}$(\texttt{monkey}) = 0.2,
        $w_{p}$(\texttt{zoo}) = 0.05,
        $w_{p}$(\texttt{x}) = 0, 
        $w_{p}$(\texttt{y}) = 0.
\end{itemize}

\noindent \textbf{Claim}: \textit{``The dog and monkey have a playful interaction.''} can be formalized $\conc(A_2) = \{\beta_2\}$ with
\begin{itemize}
    \item $\beta_{2} = $ \texttt{PlayfulInteraction(dog, monkey)}
\end{itemize}
\noindent The weights according to $A_1$ and $A_3$ are distributed as follows:
\begin{itemize}
    \item $w_{c}(\beta_{2}) = 1$;
    \item $w_{p}$(\texttt{PlayfulInteraction}) = 0.7;
    \item $w_{p}$(\texttt{dog}) = 0.15, $w_{p}$(\texttt{monkey}) = 0.15.
\end{itemize}

\paragraph{Argument 3: Social Conflict Interpretation}
\begin{center} \it
    A dog is teasing a monkey at the zoo, and the monkey is teasing the dog at the zoo. Mutual teasing reflects a social conflict. So the dog and the monkey may be engaged in a social conflict.
\end{center}
\textbf{Premises:} \textit{``A dog is teasing a monkey at the zoo, and the monkey is teasing the dog at the zoo. Mutual teasing reflects a social conflict.''} can be formalized $\supp(A_3) = \{\alpha_{10},\alpha_{11},\alpha_{12},\alpha_{13},\alpha_{14}\}$ 
\begin{itemize}
    \item $\alpha_{10} = $ \texttt{AtLocation(dog, zoo)};
    \item $\alpha_{11} = $ \texttt{AtLocation(monkey, zoo)};
    \item $\alpha_{12} = $ $\neg$ \texttt{Tease(x, y)} $\lor$ $\neg$ \texttt{Tease(y, x)} $\lor$ \texttt{SocialConflict(x, y)};
    \item $\alpha_{13} = $ \texttt{Tease(dog, monkey)};
    \item $\alpha_{14} = $ \texttt{Tease(monkey, dog)}.
\end{itemize}

\noindent The weights according to $A_2$ are distributed as follows:
\begin{itemize}
    \item $w_{c}(\alpha_{10}) = w_{c}(\alpha_{11}) = 0.05$;
    \item $w_{c}(\alpha_{12}) = w_{c}(\alpha_{13}) = w_{c}(\alpha_{14}) = 0.3$;
    \item $w_{p}$(\texttt{AtLocation}) = 0.05, $w_{p}$(\texttt{Tease}) = 0.05, $w_{p}$($\neg$\texttt{Tease}) = 0.05, $w_{p}$(\texttt{SocialConflict}) = 0.4;
    \item $w_{p}$(\texttt{dog}) = 0.2,
        $w_{p}$(\texttt{monkey}) = 0.2,
        $w_{p}$(\texttt{zoo}) = 0.05,
        $w_{p}$(\texttt{x}) = 0, 
        $w_{p}$(\texttt{y}) = 0.
\end{itemize}

\noindent \textbf{Claim}: \textit{``So the dog and monkey may be engaged in a social conflict.''} can be formalized as $\conc(A_3) = \{\beta_3\}$ with
\begin{itemize}
    \item $\beta_3 = $ \texttt{SocialConflict(dog, monkey)}
\end{itemize}
\noindent The weights according to $A_2$ are distributed as follows:
\begin{itemize}
    \item $w_{c}(\beta_{3}) = 1$;
    \item $w_{p}$(\texttt{SocialConflict}) = 0.7;
    \item $w_{p}$(\texttt{dog}) = 0.15, $w_{p}$(\texttt{monkey}) = 0.15.
\end{itemize}

\paragraph{Similarity Between Argument Components}

We compute the similarity between arguments by independently comparing their premises and claims. The details of all calculations are provided in the code included in the supplementary material. \\

\noindent Let us first recall the exact similarity model used in this experiment: $\bM^{\dice,\avg}_{sb,0.8} = \tuple{\simP, \simL, \simC, \simS}$ where
\begin{itemize}
    \item $\simP = \simsbert$;
    \item $\simL = \simL^{\text{weight},0.8}$; 
    \item $\simC = \text{Tve}^{\dice,\oplus^{\max}_{\simL}}$;
    \item $\simS = \simS^{\bm}$, where 
    $\simC^{\text{flat}}$ uses $\simP = \simsbert$ within $\simL^{\text{flat}, 0.8}$.
\end{itemize}

\begin{center}
    \textbf{Argument 1 vs Argument 2}
\end{center} 
We first compute the similarity between the premises of $A_1$ and $A_2$.
$$\simS(\supp(A_1), \supp(A_2)) = 0.795$$
\begin{figure}[h!]
    \centering
    \includegraphics[width=0.95\linewidth]{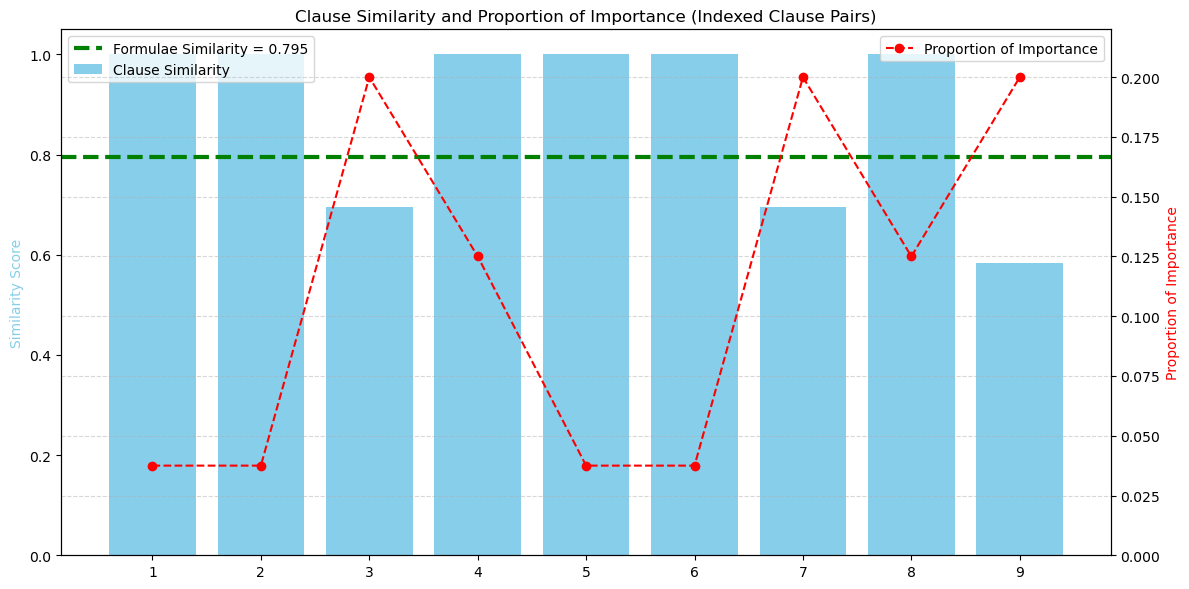}
    \caption{Similarity on the premises of $A_{1}$ vs $A_{2}$ where $1 = \alpha_1$, $2 = \alpha_2$, $3 = \alpha_3$, $4 = \alpha_4$, $5 = \alpha_5$, $6 = \alpha_6$, $7 = \alpha_7$, $8 = \alpha_8$ and $9 = \alpha_9$ on the x-axis.}
    \end{figure}

\noindent We compute the similarity between the claims of $A_1$ and $A_2$.
$$\simS(\conc(A_1), \conc(A_2)) = 0.757$$
\begin{figure}[h!]
    \centering
    \includegraphics[width=0.95\linewidth]{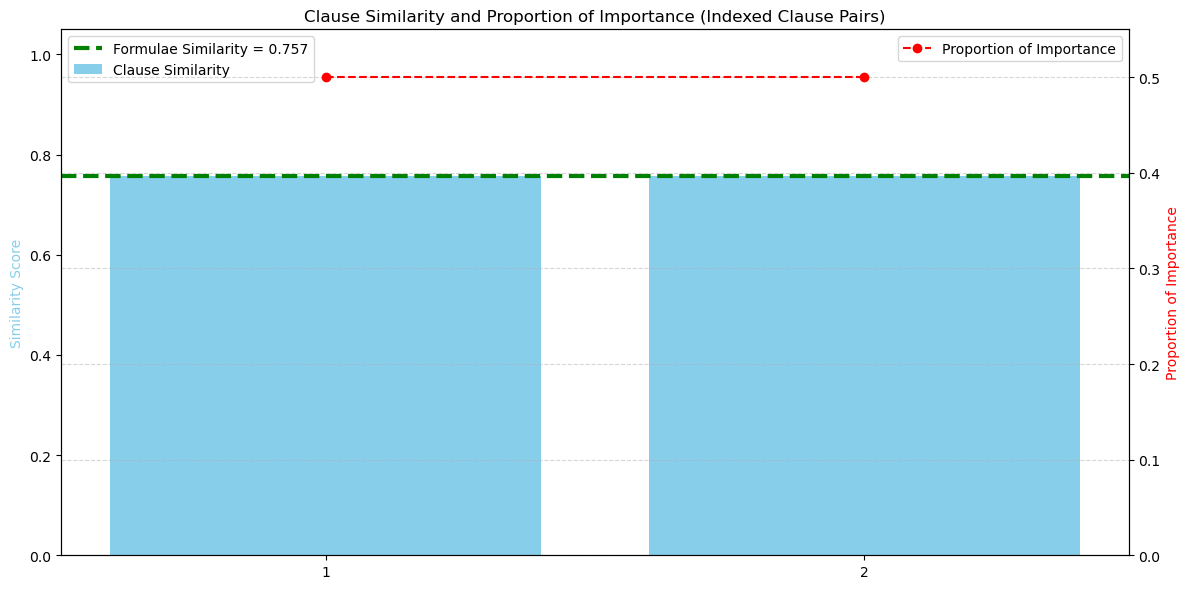}
    \caption{Similarity on the claims of $A_1$ vs $A_2$ where $1 = \beta_1$ and $2 = \beta_2$ on the x-axis.}
\end{figure}
    
\noindent Finally, we compute the similarity between $A_1$ and $A_2$ using two values of $\eta \in \{0.2,0.5\}$:
    $$\similarity\mathtt{Arg}^{\bM}_{0.5}(A_1, A_2) = 0.5 \times 0.795 + 0.5 \times 0.757 = \textbf{0.776}$$
    $$\similarity\mathtt{Arg}^{\bM}_{0.2}(A_1, A_2) = 0.2 \times 0.795 + 0.8 \times 0.757 = \textbf{0.7646}$$
    
Since the similarity scores between the premises and the claims of the two arguments are close, changing the value of $\eta$ will not have a significant impact on the similarity between $A_1$ and $A_2$ (see Figure \ref{fig:influence-eta-appendix}).

\begin{figure}[htb]
\centering
\begin{tikzpicture}[domain=0:1,scale=4]
  \draw[step=0.2,very thin,color=gray] (-0.05,-0.05) grid (1,1);

  \draw[->] (0,0) -- (1.1,0) node[right] {$\eta$};
  \draw[->] (0,0) -- (0,1.1) node[above] {$\mathtt{simArg}$};

\node at (0,-0.1) {$0$};
\node at (0.2,-0.1) {$0.2$};
\node at (0.4,-0.1) {$0.4$};
\node at (0.6,-0.1) {$0.6$};
\node at (0.8,-0.1) {$0.8$};
\node at (1,-0.1) {$1$};

\node at (-0.1,0) {$0$};
\node at (-0.1,0.2) {$0.2$};
\node at (-0.1,0.4) {$0.4$};
\node at (-0.1,0.6) {$0.6$};
\node at (-0.1,0.8) {$0.8$};
\node at (-0.1,1) {$1$};

  \draw[color=red]    plot (\x,{\x*0.795 + (1-\x)*0.757}) node[right] {$\mathtt{simArg}(A_1,A_2)$};
  \draw[color=blue]   plot (\x,{\x*0.913 + (1-\x)*0.653})    node[right] {$\mathtt{simArg}(A_2,A_3)$};
\end{tikzpicture}
\caption{Influence of $\eta$ on argument similarity: higher values prioritize the support, lower values the claim. \label{fig:influence-eta-appendix}}
\end{figure}
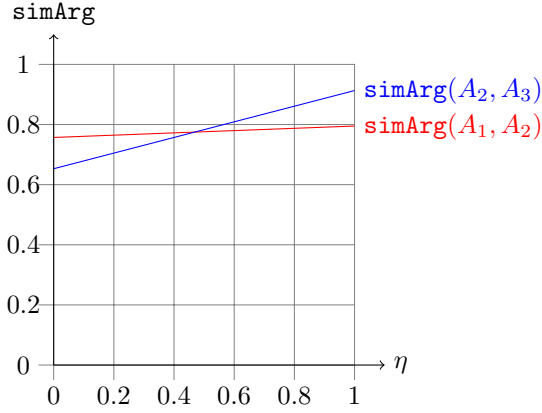

\begin{center}
    \textbf{Argument 2 vs Argument 3}
\end{center} 
We first compute the similarity between the premises of $A_2$ and $A_3$.
$$\simS(\supp(A_2), \supp(A_3)) = 0.913$$
\begin{figure}[h!]
    \centering
    \includegraphics[width=0.95\linewidth]{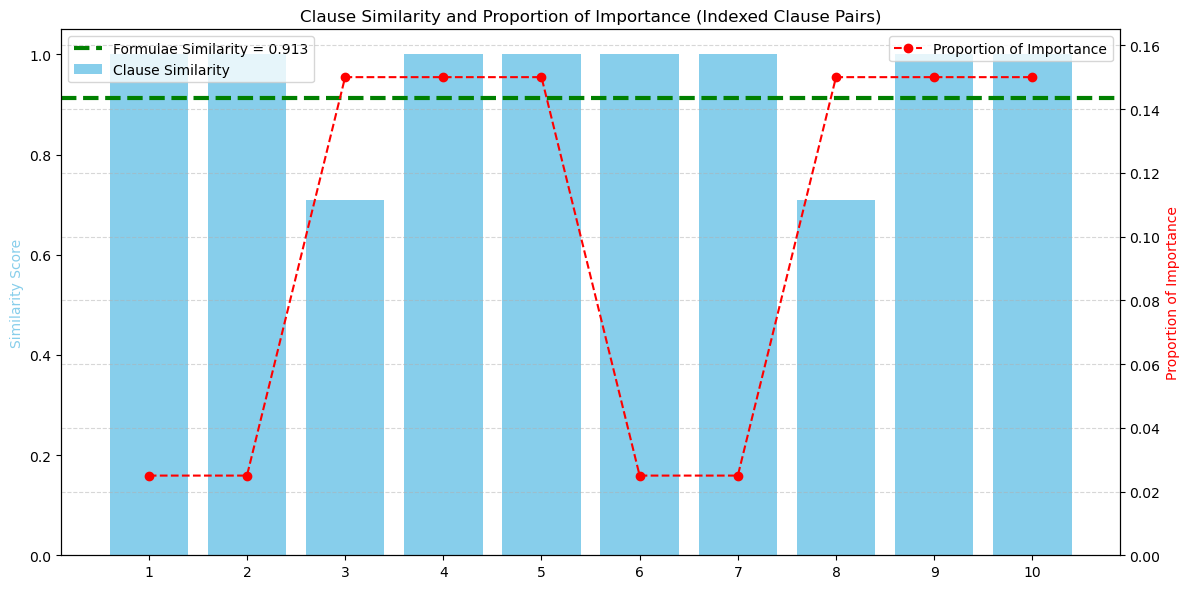}
    \caption{Similarity on the premises of $A_2$ vs $A_3$ where $1 = \alpha_5$, $2 = \alpha_6$, $3 = \alpha_7$, $4 = \alpha_8$, $5 = \alpha_9$, $6 = \alpha_{10}$, $7 = \alpha_{11}$, $8 = \alpha_{12}$, $9 = \alpha_{13}$ and $10 = \alpha_{14}$ on the x-axis.}
\end{figure}

\noindent We compute the similarity between the claims of $A_2$ and $A_3$.
$$\simS(\conc(A_2), \conc(A_3)) = 0.653$$
\begin{figure}[h!]
    \centering
    \includegraphics[width=0.95\linewidth]{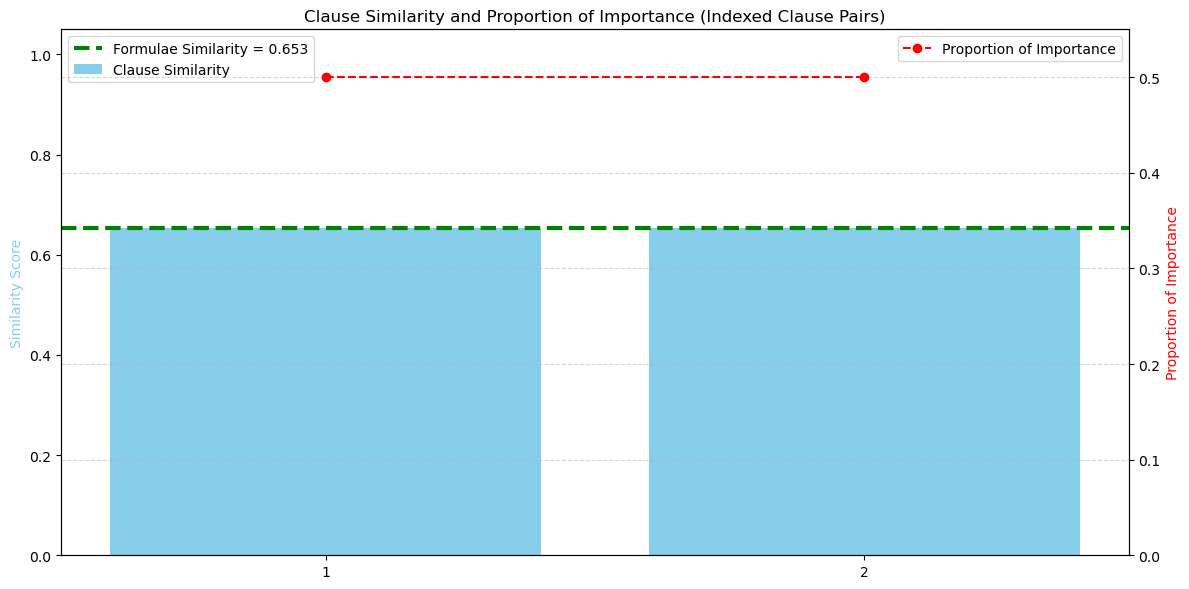}
    \caption{Similarity on the claims of $A_2$ vs $A_3$ where $1 = \beta_2$ and $2 = \beta_3$ on the x-axis.}
\end{figure}
    
\noindent Finally, we compute the similarity between $A_1$ and $A_2$ using two values of $\eta \in \{0.2,0.5\}$:
    $$\similarity\mathtt{Arg}^{\bM}_{0.5}(A_2, A_3) = 0.5 \times 0.913 + 0.5 \times 0.653 = \textbf{0.783}$$
    $$\similarity\mathtt{Arg}^{\bM}_{0.2}(A_2, A_3) = 0.2 \times 0.913 + 0.8 \times 0.653 = \textbf{0.705} $$

This example shows that our logic-based similarity framework can effectively compare structured arguments. It highlights how the aggregation of clause-level similarities, modulated by importance weights, can capture nuanced differences in argumentative structure and conclusions.

In future work, we plan to further refine the weighting schemes and investigate task-specific strategies for argument-level aggregation.\\

\vspace{0.5cm}

{\Large \noindent \textbf{Proofs.}}

\begin{proof}[\textbf{Proof (Theorem \ref{theo:wbMax})}]
\upshape{
Let a similarity model $\bM = \tuple{\simP, \simL, \simC, \simS}$ such that:
\begin{enumerate}
    \item $\simP(t,t) = 1$.
    \item $\simL((P, \vec{a}), (Q, \vec{b})) = 1$ if $\simP(P,Q) = 1$, $\card{\vec{a}} = \card{\vec{b}}$, and $\forall a_i \in \vec{a}, \forall b_i \in \vec{b}$, $\simP(a_i,b_i) = 1$.
    \item $\simC(C_1,C_2) = 1$ if $\forall l_1 \in C_1,\exists l_2 \in C_2$ s.t. $\simL(l_1,l_2) = 1$ and $\forall l_2 \in C_2,\exists l_1 \in C_1$ s.t. $\simL(l_2,l_1) = 1$.
    \item $\simS(\Phi_1,\Phi_2) = 1$ if $\forall C_1 \in \Phi_1,\exists C_2 \in \Phi_2$ s.t. $\simC(C_1, C_2) = 1$ and $\forall C_2 \in \Phi_2,\exists C_1 \in \Phi_1$ s.t. $\simC(C_2, C_1) = 1$.
\end{enumerate}

From conditions 1 and 2, we obtain: $\simL((P, \vec{a}), (P, \vec{a})) = 1$.
Together with condition 3, this implies: $\simC(C_1, C_1) = 1$.
Finally, applying condition 4, we deduce: $\simS(\Phi_1, \Phi_1) = 1$.

According to Definition~\ref{def:smArguments}, for any FOL arguments $A$ and $B$, the argument similarity measure is defined as: $\similarity\mathtt{Arg}^{\bM}_{\eta}(A, B) =$\\
$ \eta \times \simS(\supp(A^{\mathtt{c}}), \supp(B^{\mathtt{c}})) + (1 - \eta) \times \simS(\conc(A^{\mathtt{c}}), \conc(B^{\mathtt{c}})).$\\
Since $\simS(\supp(A^{\mathtt{c}}), \supp(A^{\mathtt{c}})) = 1$ and $\simS(\conc(A^{\mathtt{c}}), \conc(A^{\mathtt{c}})) = 1$, it follows that for any $\eta \in (0, 1)$,
\[
\similarity\mathtt{Arg}^{\bM}_{\eta}(A, A) = 1,
\]
i.e., $\similarity\mathtt{Arg}^{\bM}_{\eta}$ satisfies the principle \textbf{Maximality}.

}
\end{proof}


\begin{proof}[\textbf{Proof (Theorem \ref{theo:wbSym})}]
\upshape{

Let a similarity model $\bM = \tuple{\simP, \simL, \simC, \simS}$ such that $\simP, \simL, \simC,$ and $\simS$ are symmetric.

According to Definition~\ref{def:smArguments}, for any FOL arguments $A$ and $B$, the argument similarity measure is defined as: $\similarity\mathtt{Arg}^{\bM}_{\eta}(A, B) =$\\
$ \eta \times \simS(\supp(A^{\mathtt{c}}), \supp(B^{\mathtt{c}})) + (1 - \eta) \times \simS(\conc(A^{\mathtt{c}}), \conc(B^{\mathtt{c}})).$

Given that $\similarity\mathtt{Arg}^{\bM}_{\eta}(A, B)$ is defined in terms of $\simS$, the symmetry constraint on $\simS$ alone would be sufficient to ensure the Symmetry principle. However, as explained in Section~\ref{sec:sim-model} Multi-Level Arguments Similarity, $\simS$ is recursively defined: it relies on $\simC$, which depends on $\simL$, which in turn is built upon $\simP$. Therefore, to guarantee that $\simS$ is symmetric, each component function of the similarity model must also be symmetric.

Under these conditions, it follows that for any $\eta \in (0, 1)$, $\similarity\mathtt{Arg}^{\bM}_{\eta}$ satisfies \textbf{Symmetry}.

}
\end{proof}


\begin{proof}[\textbf{Proof (Theorem \ref{theo:wbSub})}]
\upshape{
Let a similarity model $\bM = \tuple{\simP, \simL, \simC, \simS}$.
Recall that from Definition~\ref{def:smArguments}, for any FOL arguments $A$ and $B$, the argument similarity measure is defined as: $\similarity\mathtt{Arg}^{\bM}_{\eta}(A, B) =$\\
$ \eta \times \simS(\supp(A^{\mathtt{c}}), \supp(B^{\mathtt{c}})) + (1 - \eta) \times \simS(\conc(A^{\mathtt{c}}), \conc(B^{\mathtt{c}})).$

Let $\eta \in (0, 1)$, and $A,B,C \in \Arg$ such that $\simi^{\bM}_{\eta}(A,B) = 1$ then from  Definition~\ref{def:smArguments}, $\simS(\supp(A^{\mathtt{c}}), \supp(B^{\mathtt{c}})) = 1$ and $\simS(\conc(A^{\mathtt{c}}), \conc(B^{\mathtt{c}})) = 1$. 

Let $\bM$ such that if $\simS(\Phi_1, \Phi_2) = 1$ then $\simS(\Phi_1,\Phi_3) = \simS(\Phi_2,\Phi_3)$.
Hence from this condition, 
$\simS(\supp(A^{\mathtt{c}}), \supp(C^{\mathtt{c}})) = \simS(\supp(B^{\mathtt{c}}), \supp(C^{\mathtt{c}}))$ and $\simS(\conc(A^{\mathtt{c}}), \conc(C^{\mathtt{c}})) = \simS(\conc(B^{\mathtt{c}}), \conc(C^{\mathtt{c}}))$.
It follows that 
$\simi(A,C) = \simi(B,C)$, i.e., 
$\similarity\mathtt{Arg}^{\bM}_{\eta}$  satisfies \textbf{Substitution}.
}
\end{proof}


\begin{proof}[\textbf{Proof (Theorem \ref{theo:wbSI})}]

Assume that $\simP(t_1,t_1) = 1$ and $\simP(t_1,t_2) = 0$ for any $t_1, t_2$. Consider four predicates or four terms $t_a, t_b, t_a', t_b'$ such that there is a bijective renaming function $\pi$ with $\pi(t_a) =  t_a'$ and $\pi(t_b) = t_b'$. It is straightforward that $\simP(t_a,t_b) = 1$ if $t_a = t_b$ and $0$ otherwise, and similarly $\simP(t_a',t_b') = 1$ if $t_a' = t_b'$ and $0$ otherwise, so syntax independence is satisfied at the level of predicates and terms.

We know that $\simL$ is based on $\simP$, formally for any literals $l_1 = P_1(\vec{x}), l_2 = P_2(\vec{y})$,  $\simL(l_1, l_2) = f(\simP(P_1,P_2), \{\simP(x_i, y_j) \mid x_i \in \vec{x}, y_j \in \vec{y}\},$ $\{\simP(y_j, x_i) \mid x_i \in \vec{x}, y_j \in \vec{y}\})$ where $P_1(\vec{x})$ and $P_2(\vec{y})$ are respectively the atoms underlying $l_1$ and $l_2$, and $f(\cdot)$ is some function computing the similarity between literals from the similarity between the predicates and terms (where $f(\cdot)$ may ignore some pairs $(x_i, y_j)$ or $(y_j,x_i)$ in the computation).

Now, let $l_a, l_b, l_a', l_b'$ be four literals with $\pi(l_a) = l_a'$ and $\pi(l_b) = l_b'$. Since we already know that syntax independence is satisfied at the level of predicates and terms, we have $\simL(l_a, l_b) = f(\simP(P_a,P_b), \{\simP(x_i, y_j) \mid x_i \in \vec{x}, y_j \in \vec{y}\}, \{\simP(y_j, x_i) \mid x_i \in \vec{x}, y_j \in \vec{y}\}) = f(\simP(P_a',P_b'), \{\simP(x_i', y_j') \mid x_i' \in \vec{x}', y_j' \in \vec{y}'\}, \{\simP(y_j', x_i') \mid x_i' \in \vec{x}', y_j' \in \vec{y}'\}) = \simL(l_a', l_b')$.

Continuing this reasoning, we know that $\simC$ is based on $\simL$, {\em i.e.} the similarity between clauses is based on the similarity between the literals that are in these clauses. Formally, for any clauses $C_1, C_2$, $\simC(C_1, C_2) = g(\{(l_1,l_2) \mid l_1 \in C_1, l_2 \in C_2\})$ where $g(\cdot)$ is some function computing the similarity between clauses from the similarity between all the pairs of literals in the clauses (notice that $g(\cdot)$ may ignore some pairs of literals $(l_1, l_2)$).

Now, given $C_a, C_b, C_a', C_b'$ four clauses such that $\pi(C_a) = C_a'$ and $\pi(C_b) = C_b'$, since syntax independence is satisfied at the level of literals, we know that $\simC(C_1, C_2) = g(\{(l_1,l_2) \mid l_1 \in C_1, l_2 \in C_2\}) = g(\{(l_1',l_2') \mid l_1' \in C_1', l_2' \in C_2'\}) = \simC(C_1', C_2')$.

Next, this reasoning can be adapted to the computation of the similarity between sets of clauses: the similarity $\simS$ between two sets of clauses is based on the similarity between the clauses in these sets. Formally, for any $\Phi_1, \Phi_2$, $\simS(\Phi_1, \Phi_2) = h(\{(C_1, C_2) \mid C_1 \in \Phi_1, C_2 \in \Phi_2\})$ where $h(\cdot)$ is a function computing the similarity between two sets of clauses from the similarity between all the pairs of clauses in these sets (again, some pairs $(C_1,C_2)$ may be ignored by the function $h(\cdot)$).

Now, considering four sets of clauses $\Phi_a, \Phi_b, \Phi_a', \Phi_b'$ such that $\pi(\Phi_a) = \Phi_a'$ and $\pi(\Phi_b) = \Phi_b'$, we have $\simS(\Phi_1, \Phi_2) = h(\{(C_1, C_2) \mid C_1 \in \Phi_1, C_2 \in \Phi_2\}) = h(\{(C_1', C_2') \mid C_1' \in \Phi_1', C_2' \in \Phi_2'\}) = \simS(\Phi_1', \Phi_2')$, so syntax independence is satisfied at the level of sets of clauses.

To prove that the principle is satisfied at the level of arguments, we simply need to apply Definition~\ref{def:smArguments}. For any arguments $A, B, A', B'$ such that $\pi(A) = A'$ and $\pi(B) = B'$,
\[
\begin{array}{rcl}
    \similarity\mathtt{Arg}^{\bM}_{\eta}(A, B) & = & \eta \times \simS(\supp(A^{\mathtt{c}}), \supp(B^{\mathtt{c}})) \\
    & & + (1-\eta) \times \simS(\conc(A^{\mathtt{c}}), \conc(B^{\mathtt{c}}))\\
    & = & \eta \times \simS(\supp(A^{\mathtt{c}'}), \supp(B^{\mathtt{c}'})) \\
    & & + (1-\eta) \times \simS(\conc(A^{\mathtt{c}'}), \conc(B^{\mathtt{c}'})) \\
    & = & \similarity\mathtt{Arg}^{\bM}_{\eta}(A', B') 
\end{array}
\]
This concludes the proof that \textbf{Syntax Independence} is satisfied by $\similarity\mathtt{Arg}^{\bM}_{\eta}$.

\end{proof}


\begin{proof}[\textbf{Proof (Theorem \ref{theo:wbMin})}]
\upshape{
Let a similarity model $\bM = \tuple{\simP, \simL, \simC, \simS}$. 
Let $A, B \in \Arg$ be two arguments such that  $A^\mathtt{c}, B^\mathtt{c} \in \ArgF$ and:
\begin{enumerate}
    \item[a.] $A^\mathtt{c}$ and $B^\mathtt{c}$ are not trivial,
    \item[b.] $\forall C_1 \in \supp(A^\mathtt{c}), \forall (P_1, \vec{a_1}) \in C_1, \forall a_1^i \in \vec{a_1}, 
    \forall C_2 \in \supp(B^\mathtt{c}), \forall (P_2, \vec{a_2}) \in C_2, \forall a_2^j \in \vec{a_2}$, $\simP(P_1,P_2) = 0$ and $\simP(a^i_1,a^j_2) = 0$,
    \item[c.] $\forall C_1 \in \conc(A^\mathtt{c}), \forall (P_1, \vec{a_1}) \in C_1, \forall a_1^i \in \vec{a_1}, 
    \forall C_2 \in \conc(B^\mathtt{c}), \forall (P_2, \vec{a_2}) \in C_2, \forall a_2^j \in \vec{a_2}$, $\simP(P_1,P_2) = 0$ and $\simP(a^i_1,a^j_2) = 0$.
\end{enumerate}
 
Assume that $\bM$ satisfies the three hypotheses from Theorem~\ref{theo:wbMin}, {\em i.e.}: 
\begin{enumerate}
    \item $\simL((P, \vec{a}), (Q, \vec{b})) = 0$ if $\simP(P,Q) = 0$ and $\forall a_i \in \vec{a},\forall b_j \in \vec{b}$, $\simP(a_i,b_j) = 0$.
    \item {\midsize $\simC(C_1,C_2) = 0$ if $\forall l_1 \in C_1,\forall l_2 \in C_2$, $\simL(l_1,l_2) = 0$.}
    \item $\simS(\Phi_1,\Phi_2) = 0$ if $\forall C_1 \in \Phi_1,\forall C_2 \in \Phi_2$, $\simC(C_1, C_2) = 0$.
\end{enumerate}

From condition a., we know that the supports of both $A^{\mathtt{c}}$ and $B^{\mathtt{c}}$ are non-empty, i.e., $\supp(A^{\mathtt{c}}) \neq \emptyset$ and $\supp(B^{\mathtt{c}}) \neq \emptyset$.

Then, by condition b. (respectively, condition c.), we know that every possible pair of predicates or terms appearing in the clauses of the supports (respectively, of the claims) of $A^{\mathtt{c}}$ and $B^{\mathtt{c}}$ have a similarity of $0$ under $\simP$.

It follows that all pairs of literals formed from the clauses in the supports (respectively, claims) of $A^{\mathtt{c}}$ and $B^{\mathtt{c}}$ also have similarity $0$ under $\simL$, since the literal-level similarity depends solely on the similarity between their predicates and terms (the formal definition of $\simL$ based on the function $f$ using $\simP$ is provided in the proof of Theorem~\ref{theo:wbSI}).

Consequently, every possible pair of clauses between the supports (respectively, claims) of $A^{\mathtt{c}}$ and $B^{\mathtt{c}}$ have similarity $0$ under $\simC$ (the formal definition of $\simC$ based on the function $g$ using $\simL$ is provided in the proof of Theorem~\ref{theo:wbSI}).

Hence, the similarity between the support $\supp(A^{\mathtt{c}})$ and $\supp(B^{\mathtt{c}})$, and between the claims $\conc(A^{\mathtt{c}})$ and $\conc(B^{\mathtt{c}})$, is also $0$ under $\simS$ (the formal definition of $\simS$ based on the function $h$ using $\simC$ is provided in the proof of Theorem~\ref{theo:wbSI}).

Finally, by Definition~\ref{def:smArguments}, since $\simS(\supp(A^{\mathtt{c}}), \supp(B^{\mathtt{c}})) = 0$ and $\simS(\conc(A^{\mathtt{c}}), \conc(B^{\mathtt{c}})) = 0$, it follows that for any $\eta \in (0,1)$,
\[
\similarity\mathtt{Arg}^{\bM}_{\eta}(A, B) = 0.
\]
Thus, $\similarity\mathtt{Arg}^{\bM}_{\eta}$ satisfies the principle \textbf{Minimality}.
}
\end{proof}

\begin{proof}[\textbf{Proof (Theorem \ref{theo:wbNZ})}]
Let $A, B \in \Arg$ be two arguments such that $A^\mathtt{c}, B^\mathtt{c} \in \ArgF$ and:
\begin{itemize}
    \item[a.] $A^\mathtt{c}$ and $B^\mathtt{c}$ are not trivial;
    \item[b.] There exist $C_1 \in \supp(A^\mathtt{c})$ and $C_2 \in \supp(B^\mathtt{c})$ such that $\exists (P_1, \vec{a_1}) \in C_1$ and $(P_2, \vec{a_2}) \in C_2$ satisfying:
    \begin{itemize}
        \item either $\simP(P_1, P_2) > 0$; 
        \item or there exists $i$ such that $\simP(a^i_1, a^i_2) > 0$. 
    \end{itemize}
\end{itemize}

Assuming that $\bM$ satisfies the three conditions of Theorem~\ref{theo:wbNZ}, we proceed as follows:

\begin{enumerate}
    \item By condition 1 and fact (b), we have $\simL((P_1, \vec{a_1}), (P_2, \vec{a_2})) > 0$.
    
    \item By condition 2, it follows that $\simC(C_1, C_2) > 0$.

    \item By condition 3, since we assume $w_c(C_1), w_c(C_2) > 0$, it follows that:
    \[
    \simS(\supp(A^\mathtt{c}), \supp(B^\mathtt{c})) > 0
    \]
\end{enumerate}

Finally, by Definition~\ref{def:smArguments}, this implies:
\[
\similarity\mathtt{Arg}^{\bM}_{\eta}(A,B) > 0
\]
which concludes the proof that $\similarity\mathtt{Arg}^{\bM}_{\eta}$ satisfies \textbf{Non-Zero} under the stated conditions.
\end{proof}

\begin{proof}[\textbf{Proof (Theorem \ref{theo:wbMon})}]
\upshape{
Let a similarity model $\bM = \tuple{\simP, \simL, \simC, \simS}$, two set of clauses $\Phi, \Psi$ and two clauses $C_0,C_1$, such that  
$\simS(\Phi \cup \{C_0\},\Psi) \leq \simS(\Phi,\Psi) \leq \simS(\Phi \cup \{C_1\},\Psi)$ when $\max_{C \in \Psi} \simC(C,C_0) = 0$ and $\max_{C \in \Psi}  \simC(C,C_1) = 1$.

Let $A, B, C \in \Arg$ such that $A^\mathtt{c}, B^\mathtt{c}, C^\mathtt{c} \in \ArgF$, and: 
\begin{enumerate}
    \item $\supp(B^\mathtt{c}) = \supp(A^\mathtt{c}) \cup \{\beta\}$ (where $\beta$ is a new clause),
    \item $\simS(\conc(A^\mathtt{c}),\conc(C^\mathtt{c})) = \simS(\conc(B^\mathtt{c}),\conc(C^\mathtt{c}))$; and  
\end{enumerate}
$\bullet$~(\textbf{S-Monotony$^0$}) if $\forall \alpha \in \supp(C^\mathtt{c})$, $\simC(\alpha,\beta) = 0$, then $\simi(A,C) \geq \simi(B,C)$; \\

Assume the condition $\forall \alpha \in \supp(C^{\mathtt{c}}),\ \simC(\alpha, \beta) = 0$ implies that $\max_{C \in \Psi} \simC(C, C_0) = 0$, where $\Psi = \supp(C^{\mathtt{c}})$ and $C_0 = \beta$.

As a result, we have:
\[
\simS(\Phi \cup \{C_0\}, \Psi) \leq \simS(\Phi, \Psi),
\]
where $\Phi = \supp(A^{\mathtt{c}})$ and $\supp(B^{\mathtt{c}}) = \Phi \cup \{C_0\}$.  
Hence,
\[
\simS(\supp(B^{\mathtt{c}}), \supp(C^{\mathtt{c}})) \leq \simS(\supp(A^{\mathtt{c}}), \supp(C^{\mathtt{c}})).
\]

From condition 2., since the claims of $A$ and $B$ have the same similarity with the claim of $C$, the overall similarity between arguments $A$ and $C$, and between $B$ and $C$, depends solely on the similarity between their supports.

Hence, by Definition~\ref{def:smArguments}, for any $\eta \in (0,1)$, we conclude:
\[
\similarity\mathtt{Arg}^{\bM}_{\eta}(A, C) \geq \similarity\mathtt{Arg}^{\bM}_{\eta}(B, C).
\]
which satisfies the principle S-Monotony$^0$.\\

\noindent $\bullet$~(\textbf{S-Monotony$^1$}) if $\exists \alpha \in \supp(C^\mathtt{c})$ s.t.  $\simC(\alpha,\beta) = 1$, then $\simi(A,C) \leq \simi(B,C)$.

Assume there exists a clause $\alpha \in \supp(C^{\mathtt{c}})$ such that $\simC(\alpha, \beta) = 1$, where $\beta = C_1$ is the clause added to the support of $A^{\mathtt{c}}$ to form $B^{\mathtt{c}}$.

Let $\Phi = \supp(A^{\mathtt{c}})$ and $\Psi = \supp(C^{\mathtt{c}})$, with $\supp(B^{\mathtt{c}}) = \Phi \cup \{C_1\}$.

Since there exists $\alpha \in \Psi$ such that $\simC(\alpha,C_1) = 1$, we have:
\[
\max_{C \in \Psi} \simC(C, C_1) = 1.
\]

As a result, the addition of $C_1$ to $\Phi$ can only increase or preserve the similarity score with respect to $\Psi$, hence:
\[
\simS(\Phi \cup \{C_1\}, \Psi) \geq \simS(\Phi, \Psi),
\]
which implies:
\[
\simS(\supp(B^{\mathtt{c}}), \supp(C^{\mathtt{c}})) \geq \simS(\supp(A^{\mathtt{c}}), \supp(C^{\mathtt{c}})).
\]

From condition 2., since the claims of $A$ and $B$ have the same similarity with the claim of $C$, the overall similarity between arguments $A$ and $C$, and between $B$ and $C$, depends solely on the similarity between their supports.

Therefore, by Definition~\ref{def:smArguments}, for any $\eta \in (0,1)$, we conclude:
\[
\similarity\mathtt{Arg}^{\bM}_{\eta}(A, C) \leq \similarity\mathtt{Arg}^{\bM}_{\eta}(B, C),
\]
which satisfies the principle S-Monotony$^1$.\\

Let us see the similar cases with the conditions of the C-Monotony.

\begin{enumerate}
    \item $\conc(B^\mathtt{c}) = \conc(A^\mathtt{c}) \cup \{\beta\}$ (where $\beta$ is a new clause),
    \item $\simS(\supp(A^\mathtt{c}),\supp(C^\mathtt{c})) = \simS(\supp(B^\mathtt{c}),\supp(C^\mathtt{c}))$; and either 
\end{enumerate}
$\bullet$~(\textbf{C-Monotony$^0$}) if $\forall \alpha \in \conc(C^\mathtt{c})$, $\simC(\beta,\alpha) = 0$, then $\simi(A,C) \geq \simi(B,C)$; or \\
\noindent $\bullet$~(\textbf{C-Monotony$^1$}) if $\exists \alpha \in \conc(C^\mathtt{c})$ s.t. $\simC(\beta,\alpha) = 1$, then $\simi(A,C) \leq \simi(B,C)$.

The satisfaction of the \textbf{C-Monotony}$^0$ and \textbf{C-Monotony}$^1$ principles follows from the same reasoning as for \textbf{S-Monotony}$^0$ and \textbf{S-Monotony}$^1$, by structural symmetry. 

Indeed, the similarity $\similarity\mathtt{Arg}^{\bM}_{\eta}$ is defined as a  combination of the similarity between supports and the similarity between claims. In the present case, conditions ensure that the supports of $A^{\mathtt{c}}$ and $B^{\mathtt{c}}$ are fixed and equal in their similarity to $C^{\mathtt{c}}$.

Hence, the variation in $\similarity\mathtt{Arg}^{\bM}_{\eta}(A,C)$ and $\similarity\mathtt{Arg}^{\bM}_{\eta}(B,C)$ depends only on the similarity between the claims. Adding a clause $\beta$ to the claim of $A^{\mathtt{c}}$ to form $B^{\mathtt{c}}$ either strictly decreases similarity when all pairwise clause similarities with $\conc(C^{\mathtt{c}})$ are zero (C-Monotony$^0$), or increases or preserves it when at least one such similarity is maximal (C-Monotony$^1$).

Therefore, the \textbf{C-Monotony} principles are satisfied by the similarity model.
}
\end{proof}

\begin{proof}[\textbf{Proof (Theorem~\ref{theo:wbRein})}]
\upshape{
Let $\bM = \tuple{\simP, \simL, \simC, \simS}$ be a similarity model satisfying the following property:

For any sets of clauses $\Phi$ and $\Psi$, and for any clauses $C_A$ and $C_B$, if:
\[
\forall C \in \Psi,\ \simC(C_A,C) \geq \simC(C_B,C),
\]
then:
\[
\simS(\Phi \cup \{C_A\}, \Psi) \geq \simS(\Phi \cup \{C_B\}, \Psi),
\]
and if additionally:
\[
\exists C \in \Psi \text{ such that } \simC(C_A,C) > \simC(C_B,C),
\]
then:
\[
\simS(\Phi \cup \{C_A\}, \Psi) > \simS(\Phi \cup \{C_B\}, \Psi).
\]

We now prove that this condition implies the satisfaction of the four reinforcement principles for $\similarity\mathtt{Arg}^{\bM}_{\eta}$.

\medskip

\textbf{S-Reinforcement$^\geq$}.

Let $A, B, C \in \Arg$ such that $A^{\mathtt{c}}, B^{\mathtt{c}}, C^{\mathtt{c}} \in \ArgF$, and:
\begin{enumerate}
    \item $\supp(A^{\mathtt{c}}) = \Phi \cup \{C_A\}$,
    \item $\supp(B^{\mathtt{c}}) = \Phi \cup \{C_B\}$,
    \item $\simS(\conc(A^{\mathtt{c}}), \conc(C^{\mathtt{c}})) = \simS(\conc(B^{\mathtt{c}}), \conc(C^{\mathtt{c}}))$,
    \item $\forall \psi \in \supp(C^{\mathtt{c}})$, $\simC(C_A, \psi) \geq \simC(C_B, \psi)$.
\end{enumerate}

Let $\Psi = \supp(C^{\mathtt{c}})$.  
By condition 4 which implies the hypothesis of the theorem, we have:
\[
\simS(\Phi \cup \{C_A\}, \Psi) \geq \simS(\Phi \cup \{C_B\}, \Psi).
\]

Since the claims of $A$ and $B$ have the same similarity to $C$, and $\similarity\mathtt{Arg}^{\bM}_{\eta}$ is computed as a combination of support and claim similarities (Definition~\ref{def:smArguments}), it follows that:
\[
\similarity\mathtt{Arg}^{\bM}_{\eta}(A, C) \geq \similarity\mathtt{Arg}^{\bM}_{\eta}(B, C),
\]
thus satisfying \textbf{S-Reinforcement$^\geq$}.

\medskip

\textbf{S-Reinforcement$^>$}.

If in addition we have $\exists \psi \in \Psi$ such that $\simC(C_A, \psi) > \simC(C_B, \psi)$, then again by hypothesis:
\[
\simS(\Phi \cup \{C_A\}, \Psi) > \simS(\Phi \cup \{C_B\}, \Psi),
\]
and hence:
\[
\similarity\mathtt{Arg}^{\bM}_{\eta}(A, C) > \similarity\mathtt{Arg}^{\bM}_{\eta}(B, C),
\]
which satisfies \textbf{S-Reinforcement$^>$}.

\medskip

\textbf{C-Reinforcement$^\geq$}.

Now consider the symmetric case where the difference lies in the claims.

Let:
\begin{itemize}
    \item $\conc(A^{\mathtt{c}}) = \Phi \cup \{C_A\}$,
    \item $\conc(B^{\mathtt{c}}) = \Phi \cup \{C_B\}$,
    \item $\simS(\supp(A^{\mathtt{c}}), \supp(C^{\mathtt{c}})) = \simS(\supp(B^{\mathtt{c}}), \supp(C^{\mathtt{c}}))$,
    \item $\forall \psi \in \conc(C^{\mathtt{c}})$, $\simC(C_A, \psi) \geq \simC(C_B, \psi)$.
\end{itemize}

Let $\Psi = \conc(C^{\mathtt{c}})$.  
By the same hypothesis:
\[
\simS(\Phi \cup \{C_A\}, \Psi) \geq \simS(\Phi \cup \{C_B\}, \Psi).
\]

Since the supports are equivalent in similarity to $C$, the result again follows:
\[
\similarity\mathtt{Arg}^{\bM}_{\eta}(A, C) \geq \similarity\mathtt{Arg}^{\bM}_{\eta}(B, C),
\]
thus satisfying \textbf{C-Reinforcement$^\geq$}.

\medskip

\textbf{C-Reinforcement$^>$}.

If in addition $\exists \psi \in \Psi$ such that $\simC(C_A, \psi) > \simC(C_B, \psi)$, then:
\[
\simS(\Phi \cup \{C_A\}, \Psi) > \simS(\Phi \cup \{C_B\}, \Psi),
\]
and consequently:
\[
\similarity\mathtt{Arg}^{\bM}_{\eta}(A, C) > \similarity\mathtt{Arg}^{\bM}_{\eta}(B, C),
\]
thus satisfying \textbf{C-Reinforcement$^>$}.
}
\end{proof}

\begin{proof}[\textbf{Proof (Theorem \ref{theo:satisfaction-principles-measures})}]
\upshape{
    Let us start by recalling the two families of instantiations of our similarity model: the first one is $\bM^{x,w_g}_{sb,\lambda} = \tuple{\simsbert, \simL, \simC, \simS}$, and the second one is $\bM^{x,w_g}_{eq,\lambda} = \tuple{\simP^{eq}, \simL, \simC, \simS}$, where:
\begin{itemize}
    \item $\simL = \simL^{\text{weight},\lambda}$ with $\lambda \in (0,1)$, and $\simP$ is set to $\simsbert$ for $\bM^{x,w_g}_{sb,\lambda}$ and to $\simP^{eq}$ for $\bM^{x,w_g}_{eq,\lambda}$; 
    \item $\simC = \text{Tve}^{x,\oplus^{\max}_{\simL}}$ with $x \in \{\jacc, \dice, \soren, \ander, \sok\}$;
    \item $\simS = \simS^{\bm}$, where 
    $\simC^{\text{flat}}$ uses $\simP = \simsbert$ for $\bM^{x,w_g}_{sb,\lambda}$ and $\simP = \simP^{eq}$ for $\bM^{x,w_g}_{eq,\lambda}$ within $\simL^{\text{flat}, \lambda}$, and $g$ is either average ($\mathtt{avg}$) or product ($\Pi$).
\end{itemize}

By proving that the first eight theorems are satisfied by our families of similarity models, we establish the satisfaction of the corresponding principles.

\begin{enumerate}
    \item $\similarity\mathtt{Arg}^{\bM^{x,w_g}_{eq,\lambda}}_{\eta}$ and $\similarity\mathtt{Arg}^{\bM^{x,w_g}_{sb,\lambda}}_{\eta}$ satisfy \textbf{Maximality}:
\begin{enumerate}
    \item $\simP(t,t) = 1$.\\
    By definition $\simP^{eq}(t,t) = 1$. \\

    By definition, $\simsbert(t_1, t_2)$ computes the cosine similarity between the SBERT embeddings of $t_1$ and $t_2$, denoted by $\vec{e}_1$ and $\vec{e}_2$. That is:
    \[
    \simsbert(t_1, t_2) = \cos(\vec{e}_1, \vec{e}_2) = \frac{\vec{e}_1 \cdot \vec{e}_2}{\|\vec{e}_1\| \cdot \|\vec{e}_2\|}
    \]
    If $t_1 = t_2 = t$, then their embeddings are identical: $\vec{e}_1 = \vec{e}_2 = \vec{e}$. Therefore:
    \[
    \simsbert(t, t) = \frac{\vec{e} \cdot \vec{e}}{\|\vec{e}\| \cdot \|\vec{e}\|} = \frac{\|\vec{e}\|^2}{\|\vec{e}\|^2} = 1
    \]
    Hence, the similarity score between two identical texts using SBERT is equal to 1.
    
    \item $\simL((P, \vec{a}), (Q, \vec{b})) = 1$ if $\simP(P,Q) = 1$, $\card{\vec{a}} = \card{\vec{b}}$, and $\forall a_i \in \vec{a}, \forall b_i \in \vec{b}$, $\simP(a_i,b_i) = 1$.

    See Lemma 2 for its proof.
    
    \item $\simC(C_1,C_2) = 1$ if $\forall l_1 \in C_1,\exists l_2 \in C_2$ s.t. $\simL(l_1,l_2) = 1$ and $\forall l_2 \in C_2,\exists l_1 \in C_1$ s.t. $\simL(l_2,l_1) = 1$.

    See Lemma 3 for its proof.
    
    \item $\simS(\Phi_1,\Phi_2) = 1$ if $\forall C_1 \in \Phi_1,\exists C_2 \in \Phi_2$ s.t. $\simC(C_1, C_2) = 1$ and $\forall C_2 \in \Phi_2,\exists C_1 \in \Phi_1$ s.t. $\simC(C_2, C_1) = 1$.

    See Lemma 4 for its proof.
    
 \end{enumerate}

    \item $\similarity\mathtt{Arg}^{\bM^{x,w_g}_{eq,\lambda}}_{\eta}$ and $\similarity\mathtt{Arg}^{\bM^{x,w_g}_{sb,\lambda}}_{\eta}$ satisfy \textbf{Substitution}:
    For $\similarity\mathtt{Arg}^{\bM^{x,w_g}_{eq,\lambda}}_{\eta}$, this directly follows from the characterization established in Theorem~\ref{theo:=1}, which states that a similarity score of~1 occurs only between equivalent arguments, and the fact that we use a normalized language $\cF$ for the compiled CNF arguments, leading to identical arguments when they are equivalent. 
    Specifically, if $\similarity\mathtt{Arg}^{\bM^{x,w_g}_{eq,\lambda}}_{\eta}(A, B) = 1$, then it must be that $A = B$. Consequently, for any argument $C$, we have:
    \[
    \similarity\mathtt{Arg}^{\bM^{x,w_g}_{eq,\lambda}}_{\eta}(A, C) = \similarity\mathtt{Arg}^{\bM^{x,w_g}_{eq,\lambda}}_{\eta}(B, C)
    \]
    as a direct result of substitutivity of identical arguments.\\
    
    For $\similarity\mathtt{Arg}^{\bM^{x,w_g}_{sb,\lambda}}_{\eta}$, we start from the theorem's condition and examine how each sub-level satisfies it:
    \begin{enumerate}
    \item {\midsize if $\simS(\Phi_1, \Phi_2) = 1$ then $\simS(\Phi_1,\Phi_3) = \simS(\Phi_2,\Phi_3)$.}

    From Definition \ref{def:simSbm}: $\simS^{\bm}(\Phi, \Psi) = $
    \begin{center}
    {\large
        $\frac{
    \sum\limits_{(C_1, C_2) \in \bm(\Phi, \Psi)}
     w_g(C_1, C_2) \times \simC(C_1, C_2)
    }{
    \sum\limits_{(C_1, C_2) \in \bm(\Phi, \Psi)}
    w_g(C_1, C_2)
    }$
    }
    \end{center}

    where $\bm$ uses a flat clause similarity measure $\simC^{\text{flat}}$ defining the best matching clause pairs as $\bm(\Phi, \Psi) = $

    {\small
    \[
    \left\{
    \begin{aligned}
    (C_1, C_2) \mid\, & \big(C_1 \in \Phi,\, C_2 = \mathtt{argmax}_{C' \in \Psi} \simC^{\text{flat}}(C_1, C')\big) \\
    \text{or } &  \big(C_1 \in \Psi,\, C_2 = \mathtt{argmax}_{C' \in \Phi} \simC^{\text{flat}}(C_1, C')\big)
    \end{aligned}
    \right\}
    \]
    }
    \noindent and where $w_g$ denotes the contribution  weight of a clause pair, computed using an aggregation function $\agg$:
    \begin{center}
        $w_g(C_1, C_2) = \agg( w_{c}(C_1),w_{c}(C_2)).$
    \end{center}

    Hence, $\simS^{\bm}(\Phi, \Psi) = 1$ iff $\forall (C_1, C_2) \in \bm(\Phi, \Psi)$, $\simC(C_1, C_2) = 1$.
    Thanks to Lemmas 1, 2, and 3, we know that when the flat clause similarity reaches a value of 1, the corresponding weighted clause similarity also equals 1, thereby ensuring that the final similarity score is 1 as well.

    \item {\midsize if $\simC(C_1, C_2) = 1$ then $\simC(C_1,C_3) = \simC(C_2,C_3)$.}\\
    $\simC^{\text{weight}}_{x}(C_1,C_2) =\text{Tve}^{x,\oplus^{\max}_{\simL^{\text{weight},\lambda}}}(C_1,C_2) = 1$ iff
    $C_1 = C_2$ as shown in Lemma 3.
    
    \item if $\simL((P_1, \vec{a_1}), (P_2, \vec{a_2})) = 1$ then $\simL((P_1, \vec{a_1}), (Q, \vec{b})) = \simL((P_2, \vec{a_2}), (Q, \vec{b}))$.\\
    $\simL^{\text{weight},\lambda}((P_1, \vec{a_1}), (P_2, \vec{a_2})) = 1$ iff $\simsbert(P, P) = 1$ and $\similarity_{\text{para}}^{\text{}\lambda}(\vec{a}, \vec{a}) = 1$ as shown in Lemma 2.

    \item if $\simP(t_1,t_2) = 1$ then $\simP(t_1,t_3) = \simP(t_2,t_3)$.\\
    $\simsbert(t_1,t_2) = 1$ iff the embeddings of $t_1$ and $t_2$ are identical, then in this case that $\simsbert(t_1,t_3) = \simsbert(t_2,t_3)$.
    
    \end{enumerate}

    \item $\similarity\mathtt{Arg}^{\bM^{x,w_g}_{eq,\lambda}}_{\eta}$ and $\similarity\mathtt{Arg}^{\bM^{x,w_g}_{sb,\lambda}}_{\eta}$ satisfy \textbf{Symmetry}, as all components of both models ($\simP$, $\simL$, $\simC$, and $\simS$) are symmetric:
    \begin{itemize}
        \item $\simP^{eq}$ and $\simsbert$ are symmetric by definition;
        \item $\simL^{\text{weight},\lambda}$ and $\simL^{\text{flat},\lambda}$ are symmetric for any $\lambda \in (0,1)$, due to the use of $\min$ in $\similarity_{\text{pos}}$ and the best-match function \(\best{x}{\vec{y}}\) in $\similarity_{\text{unord}}$;
        \item $\simC^{\text{weight},\lambda}$ and $\simC^{\text{flat},\lambda}$ are symmetric for any $\lambda \in (0,1)$, as they rely on fuzzy Tversky with symmetric Tversky parameters, i.e., $x \in \{\jacc, \dice, \soren, \ander, \sok\}$, and use $\max$ in the membership function;
        \item $\simS^\bm$ is symmetric due to the symmetric best-matching procedure $\bm$ and the use of symmetric aggregation functions (average or product) in contribution weights.
    \end{itemize}

    \item $\similarity\mathtt{Arg}^{\bM^{x,w_g}_{eq,\lambda}}_{\eta}$ satisfies \textbf{Syntax Independence}, since $\simP^{eq}(t_1, t_1) = 1$ and $\simP^{eq}(t_1, t_2) = 0$ for $t_1 \neq t_2$, making it invariant to lexical variation.

    By contrast, $\similarity\mathtt{Arg}^{\bM^{x,w_g}_{sb,\lambda}}_{\eta}$ does not satisfy \textbf{Syntax Independence}. This is illustrated in the running example with $\texttt{CNF\_T1}$ and $\texttt{CNF\_T2}$: replacing \texttt{monkey} with \texttt{puppy} changes the similarity. Specifically, $\simsbert(\texttt{dog}, \texttt{puppy}) = 0.804$ versus $\simsbert(\texttt{dog}, \texttt{monkey}) = 0.466$, resulting in a different final similarity score, $0.864$ instead of $0.628$. This highlights the model’s sensitivity to lexical choices.

    \item $\similarity\mathtt{Arg}^{\bM^{x,w_g}_{eq,\lambda}}_{\eta}$ and $\similarity\mathtt{Arg}^{\bM^{x,w_g}_{sb,\lambda}}_{\eta}$ satisfy  \textbf{Minimality}:
    \begin{enumerate}
        \item $\simL((P, \vec{a}), (Q, \vec{b})) = 0$ if $\simP(P,Q) = 0$ and $\forall a_i \in \vec{a},\forall b_j \in \vec{b}$, $\simP(a_i,b_j) = 0$.
        
        By definition, the literal-level similarity measure is given by: $\simL^{\text{weight},\lambda}((P, \vec{a}), (Q, \vec{b})) =$ 
    $
    \frac{w_p(P) w_p(Q) \times \simP(P, Q) + \similarity_{\text{para}}^{\lambda}(\vec{a}, \vec{b}) \times (\lambda w_\text{ord} + (1-\lambda) w_\text{unord})}
    {w_p(P) w_p(Q) + (\lambda w_\text{ord} + (1-\lambda) w_\text{unord})}
    $
    
    We now analyze each term under the hypothesis.
    
    \begin{itemize}
        \item Since $\simP(P, Q) = 0$, the first term in the numerator is zero:
        \[
        w_p(P) w_p(Q) \times \simP(P, Q) = 0
        \]
    
        \item For the second term, note that $\similarity_{\text{para}}^{\lambda}(\vec{a}, \vec{b})$ is an aggregation (ordered and unordered) over similarities between arguments using $\simP(a_i, b_j)$. Given that all such values are 0, it follows that:
        \[
        \similarity_{\text{para}}^{\lambda}(\vec{a}, \vec{b}) = 0
        \]
    
        \item Hence, both terms in the numerator are zero, and we obtain:
        \[
        \simL^{\text{weight},\lambda}((P, \vec{a}), (Q, \vec{b})) = \frac{0}{\text{positive denominator}} = 0
        \]
    
        \item The denominator remains strictly positive as it is a sum of positive weights (assuming $w_p(\cdot) > 0$), regardless of the similarity scores.
    \end{itemize}
    
    Therefore, the literal similarity is zero whenever both the predicate similarity and all argument similarities are zero.

        \item {\midsize $\simC(C_1,C_2) = 0$ if $\forall l_1 \in C_1,\forall l_2 \in C_2$, $\simL(l_1,l_2) = 0$.}
        
        From Definition~\ref{def:ftve}, the fuzzy Tversky similarity between two clauses is given by:
    \[
    \text{Tve}^{\alpha,\beta,\oplus^{\agg}_{\simL}}(C_1, C_2) = \frac{A}{A + \alpha B + \beta C}
    \]
    with:
    \begin{align*}
    A &= \frac{1}{2} \left( \sum_{x \in C_1} \oplus^{\agg}_{\simL}(x, C_2) + \sum_{y \in C_2} \oplus^{\agg}_{\simL}(y, C_1) \right) \\
    B &= \sum_{x \in C_1} \left(1 - \oplus^{\agg}_{\simL}(x, C_2)\right) \\
    C &= \sum_{y \in C_2} \left(1 - \oplus^{\agg}_{\simL}(y, C_1)\right)
    \end{align*}
    
    Under the assumption that $\simL(l_1, l_2) = 0$ for all $l_1 \in C_1$, $l_2 \in C_2$, it follows that:
    $
    \forall x \in C_1, \quad \oplus^{\agg}_{\simL}(x, C_2) = 0, 
    \forall y \in C_2, \quad \oplus^{\agg}_{\simL}(y, C_1) = 0
    $
    
    Thus:
    $
    A = \frac{1}{2} (0 + 0) = 0, \quad
    B = \sum_{x \in C_1} 1 = |C_1|, \quad
    C = \sum_{y \in C_2} 1 = |C_2|
    $
    
    Consequently:
    \[
    \text{Tve}^{\alpha,\beta,\oplus^{\agg}_{\simL}}(C_1, C_2) = \frac{0}{\alpha |C_1| + \beta |C_2|} = 0
    \]
        
        \item $\simS(\Phi_1,\Phi_2) = 0$ if $\forall C_1 \in \Phi_1,\forall C_2 \in \Phi_2$, $\simC(C_1, C_2) = 0$. 

        By Definition~\ref{def:simSbm}, the similarity between two sets of clauses is given by: $\simS^{\bm}(\Phi_1, \Phi_2) =$
    \[
     \frac{
    \sum\limits_{(C_1, C_2) \in \bm(\Phi_1, \Phi_2)}
    w_g(C_1, C_2) \times \simC(C_1, C_2)
    }{
    \sum\limits_{(C_1, C_2) \in \bm(\Phi_1, \Phi_2)}
    w_g(C_1, C_2)
    }
    \]
    
    By hypothesis, for all \( C_1 \in \Phi_1 \), \( C_2 \in \Phi_2 \), we have:
    \[
    \simC(C_1, C_2) = 0
    \]
    
    Therefore, regardless of how clause pairs \( (C_1, C_2) \in \bm(\Phi_1, \Phi_2) \) are selected, all similarity values in the numerator of the expression for \( \simS^{\bm} \) are zero:
    \[
    \sum\limits_{(C_1, C_2) \in \bm(\Phi_1, \Phi_2)} w_g(C_1, C_2) \times 0 = 0
    \]
    
    The denominator, being a sum of strictly positive weights (assuming the aggregation function \( \agg \) and weights \( w_c \) are non-zero for non-empty sets), remains non-zero:
    \[
    \sum\limits_{(C_1, C_2) \in \bm(\Phi_1, \Phi_2)} w_g(C_1, C_2) > 0
    \]
    
    Hence:
    \[
    \simS^{\bm}(\Phi_1, \Phi_2) = \frac{0}{\text{positive constant}} = 0
    \]

    \end{enumerate}

    \item 
    $\similarity\mathtt{Arg}^{\bM^{x,w_g}_{eq,\lambda}}_{\eta}$ and $\similarity\mathtt{Arg}^{\bM^{x,w_g}_{sb,\lambda}}_{\eta}$ satisfy \textbf{Non-Zero}, provided the sufficient conditions of Theorem~\ref{theo:wbNZ} hold:

\begin{enumerate}
    \item {\midsize $\simL((P, \vec{a}), (Q, \vec{b})) > 0$ if $w_p(P), w_p(Q) > 0$, and either $\simP(P, Q) > 0$, or there exists a position $i$ such that $a_i \in \vec{a}$, $b_i \in \vec{b}$, $\simP(a_i, b_i) > 0$, and $w_p(a_i), w_p(b_i) > 0$.}

    From the definition of the weighted literal similarity:
    $
    \simL^{\text{weight},\lambda}((P, \vec{a}), (Q, \vec{b})) =$
    $
    \frac{
    w_p(P) w_p(Q) \times \simP(P, Q) + \similarity_{\text{para}}^{\lambda}(\vec{a}, \vec{b}) \times (\lambda w_\text{ord} + (1-\lambda) w_\text{unord})
    }{
    w_p(P) w_p(Q) + (\lambda w_\text{ord} + (1-\lambda) w_\text{unord})
    }
    $

    We distinguish two cases:

    \begin{itemize}
        \item \textbf{Case 1:} $\simP(P, Q) > 0$ and $w_p(P), w_p(Q) > 0$.\\
        The first term in the numerator is strictly positive, and the denominator is also positive, so $\simL > 0$.

        \item \textbf{Case 2:} $\simP(a_i, b_i) > 0$ for some $i$, with $w_p(a_i), w_p(b_i) > 0$.\\
        Then $\similarity_{\text{para}}^\lambda(\vec{a}, \vec{b}) > 0$, and the associated weighted sum $w_\text{ord}$ or $w_\text{unord}$ is also positive. Hence, $\simL > 0$.
    \end{itemize}

    Therefore, under the stated conditions, $\simL^{\text{weight},\lambda}((P, \vec{a}), (Q, \vec{b})) > 0$.

    \item $\simC(C_1, C_2) > 0$ if $\exists l_1 \in C_1$, $l_2 \in C_2$ such that $\simL(l_1, l_2) > 0$.

    From Definition~\ref{def:ftve}:
    \[
    \simC(C_1, C_2) = \frac{A}{A + \alpha B + \beta C}
    \]
    where \( A \) includes membership values such as \(\oplus^{\agg}_{\simL}(l_1, C_2) \geq \simL(l_1, l_2) > 0 \), hence \( A > 0 \), and the denominator is also positive. Thus, $\simC(C_1, C_2) > 0$.

    \item $\simS(\Phi_1, \Phi_2) > 0$ if $\exists C_1 \in \Phi_1$, $C_2 \in \Phi_2$ such that $\simC(C_1, C_2) > 0$ and $w_c(C_1), w_c(C_2) > 0$.

    From Definition~\ref{def:simSbm}:
    $
    \simS^{\bm}(\Phi_1, \Phi_2) =$
    \[ 
    \frac{
    \sum_{(C_1, C_2) \in \bm(\Phi_1, \Phi_2)} w_g(C_1, C_2) \times \simC(C_1, C_2)
    }{
    \sum_{(C_1, C_2) \in \bm(\Phi_1, \Phi_2)} w_g(C_1, C_2)
    }
    \]
    where \( w_g(C_1, C_2) = \agg(w_c(C_1), w_c(C_2)) \).

    If one clause pair satisfies $\simC(C_1, C_2) > 0$ and $w_c(C_1), w_c(C_2) > 0$, then $w_g(C_1, C_2) > 0$ as well. Thus, both numerator and denominator are strictly positive, and we conclude:
    \[
    \simS(\Phi_1, \Phi_2) > 0
    \]
    \end{enumerate}

    \item $\similarity\mathtt{Arg}^{\bM^{x,w_g}_{eq,\lambda}}_{\eta}$ and $\similarity\mathtt{Arg}^{\bM^{x,w_g}_{sb,\lambda}}_{\eta}$ satisfy the principles of \textbf{S-Monotony$^0$}, \textbf{S-Monotony$^1$}, \textbf{C-Monotony$^0$}, and \textbf{C-Monotony$^1$}: 
     $\simS(\Phi \cup \{C_0\},\Psi) \leq \simS(\Phi,\Psi) \leq \simS(\Phi \cup \{C_1\},\Psi)$ when $\max_{C \in \Psi} \simC(C,C_0) = 0$ and $\max_{C \in \Psi}  \simC(C,C_1) = 1$.

     We use the definition of $\simS^{\bm}(\cdot, \cdot)$ as the weighted average of best matching clause similarities. When adding a clause to $\Phi$, new clause pairs are added to the set $\bm(\cdot,\cdot)$ of best matchings, and the overall similarity becomes a weighted average over a larger set.
    
    \medskip
    
    \textit{Left inequality:} Adding a clause $C_0$ such that $\simC(C, C_0) = 0$ for all $C \in \Psi$ ensures that:\\
    - Any best match involving $C_0$ contributes zero similarity;\\
    - The total weight in the denominator increases due to the inclusion of $w_g(C_0, C')$ terms, while the numerator remains unchanged or increases by zero.
    
    Hence, the weighted average either decreases or stays the same:
    \[
    \simS(\Phi \cup \{C_0\}, \Psi) \leq \simS(\Phi, \Psi)
    \]
    
    \textit{Right inequality:} Conversely, adding a clause $C_1$ such that there exists $C^\star \in \Psi$ with $\simC(C^\star, C_1) = 1$ will introduce a new pair in $\bm(\cdot,\cdot)$ with maximum similarity and positive weight. Therefore:\\
    - The numerator of $\simS$ increases by at least $w_g(C_1, C^\star)$;\\
    - The denominator increases by the same amount;\\
    - This inclusion of a highly similar clause tends to pull the average upwards.
    
    Hence:
    \[
    \simS(\Phi, \Psi) \leq \simS(\Phi \cup \{C_1\}, \Psi)
    \]
    
    \medskip
    
    Combining both bounds gives the desired inequality:
    \[
    \simS(\Phi \cup \{C_0\}, \Psi) \leq \simS(\Phi, \Psi) \leq \simS(\Phi \cup \{C_1\}, \Psi)
    \]

    \item $\similarity\mathtt{Arg}^{\bM^{x,w_g}_{eq,\lambda}}_{\eta}$ and $\similarity\mathtt{Arg}^{\bM^{x,w_g}_{sb,\lambda}}_{\eta}$ satisfy the principles of \textbf{S-Reinforcement$^\geq$}, \textbf{S-Reinforcement$^>$}, \textbf{C-Reinforcement$^\geq$}, and \textbf{C-Reinforcement$^>$}:  
    $\forall C \in \Psi,\ \simC(C_A,C) \geq \simC(C_B,C)$ then $\simS(\Phi \cup \{C_A\}, \Psi) \geq \simS(\Phi \cup \{C_B\}, \Psi)$ and 
    if also $\exists C \in \Psi \text{ s.t. } \simC(C_A,C) > \simC(C_B,C)$ then $\simS(\Phi \cup \{C_A\}, \Psi) >\simS(\Phi \cup \{C_B\}, \Psi)$.

    The similarity $\simS$ is defined as a weighted average over best-matching clause pairs in $\bm(\cdot,\cdot)$, where the best match for a clause is determined by its highest similarity to any clause in the other set.
    
    Let us compare the effect of adding $C_A$ versus $C_B$ to the set $\Phi$, when computing $\simS(\Phi \cup \{C\}, \Psi)$.
    
    \medskip
    
    \textit{Improvement:}  
    By assumption, for every $C \in \Psi$, we have:
    \[
    \simC(C_A, C) \geq \simC(C_B, C)
    \]
    This implies that:\\
    - In the best-matching clause selection $\bm(\cdot,\cdot)$, every pair involving $C_A$ will have a similarity score at least as high as the corresponding pair with $C_B$.\\
    - Since clause weights $w_c$ and aggregation function $w_g$ are identical, the contributions to the numerator from $C_A$ are at least those from $C_B$, and the denominators are equal.
    
    Therefore, the total weighted average after adding $C_A$ is at least as large as that after adding $C_B$:
    \[
    \simS(\Phi \cup \{C_A\}, \Psi) \geq \simS(\Phi \cup \{C_B\}, \Psi)
    \]
    
    \medskip
    
    \textit{Strict improvement:}  
    If there exists $C^\star \in \Psi$ such that:
    \[
    \simC(C_A, C^\star) > \simC(C_B, C^\star)
    \]
    then at least one matching pair will contribute strictly more when using $C_A$ than $C_B$. Since all other contributions are at least equal and the weight structure remains unchanged, the overall average strictly increases:
    \[
    \simS(\Phi \cup \{C_A\}, \Psi) > \simS(\Phi \cup \{C_B\}, \Psi)
    \]

\end{enumerate}

    }
\end{proof}

\begin{lemma}\label{lemma:flat-literal-simi-equals-1}
\upshape{
Let $(P, \vec{a})$ be a literal. $\simL^{\text{flat},\lambda}((P, \vec{a}), (P, \vec{a})) = 1$.
}
\end{lemma}

\begin{proof}
We recall how to compute the flat similarity between two literals $(P, \vec{a})$ and $(Q, \vec{b})$:
$$\simL^{\text{flat},\lambda}((P, \vec{a}), (Q, \vec{b})) = \frac{1}{2}(\simP(P, Q) + \similarity_{\text{para}}^{\text{}\lambda}(\vec{a}, \vec{b}))$$
where 
$\similarity_{\text{para}}^{\text{}\lambda}(\vec{a}, \vec{b}) =  \lambda \times \simFlatPos(\vec{a}, \vec{b}) + (1 - \lambda) \times \simFlatUnord(\vec{a}, \vec{b})$
with $\lambda \in [0,1]$,
  \begin{center}
      $\simFlatPos(\vec{a},\vec{b}) = \dfrac{\sum_{i=1}^{\min(|\vec{a}|, |\vec{b}|)}  \simP(a_i, b_i)}{\min(|\vec{a}|, |\vec{b}|)}$
  \end{center}
and   
  \begin{center}
      $\simFlatUnord(\vec{a},\vec{b}) = \frac{\sum\limits_{a_i \in \vec{a}} \simP(a_i, \best{a_i}{\vec{b}}) + \sum\limits_{b_j \in \vec{b}} \simP(b_j, \best{b_j}{\vec{a}})}{|\vec{a}| + |\vec{b}|}$
  \end{center}
where $\best{x}{\vec{y}} = \mathtt{arg max}_{y \in \vec{y}} \simP(x, y)$.

For our purpose, we focus on 
$\simL^{\text{flat},\lambda}((P, \vec{a}), (P, \vec{a})) = \frac{1}{2}(\simP(P, P) + \similarity_{\text{para}}^{\text{}\lambda}(\vec{a}, \vec{a}))$. We know that $\simP(t,t) = 1$ for any $t$, so $\simP(P, P) = 1$, and for the same reason, we can simply deduce that $\similarity_{\text{para}}^{\text{}\lambda}(\vec{a}, \vec{a}) = 1$ (because both fractions involved in the computation of $\simFlatPos$ and $\simFlatUnord$ reduce to $1$). We obtain $\simL^{\text{flat},\lambda}((P, \vec{a}), (P, \vec{a})) = 1$. 
\end{proof}

\begin{lemma}\label{lemma:weighted-literal-simi-equals-1}
\upshape{
Let $(P,\vec{a})$ be a literal. $\simL^{\text{weight},\lambda}((P, \vec{a}), (P, \vec{a}) = 1$.
}    
\end{lemma}

\begin{proof}
Let us recall how to compute the weighted similarity between literals:
$\simL^{\text{weight},\lambda}((P, \vec{a}), (Q, \vec{b})) =$
\begin{center}
$\frac{w_p(P) w_p(Q) \times \simP(P, Q) ~ + ~ \similarity_{\text{para}}^{\text{}\lambda}(\vec{a}, \vec{b}) \times (\lambda w_\text{ord} + (1-\lambda)  w_\text{unord})}
{w_p(P) w_p(Q) ~+~ (\lambda w_\text{ord} + (1-\lambda)  w_\text{unord})}$
\end{center}

We have already established (proof of Lemma~\ref{lemma:flat-literal-simi-equals-1}) that $\simP(P, P) = 1$ and $\similarity_{\text{para}}^{\text{}\lambda}(\vec{a}, \vec{a}) = 1$, so it is easy to see that 
\[
\begin{array}{c}
\simL^{\text{weight},\lambda}((P, \vec{a}), (P, \vec{a})) \\ \\
= \frac{w_p(P) w_p(P) \times \simP(P, P) + \similarity_{\text{para}}^{\text{}\lambda}(\vec{a}, \vec{a}) \times (\lambda w_\text{ord} + (1-\lambda)  w_\text{unord})}
{w_p(P) w_p(P) + (\lambda w_\text{ord} + (1-\lambda)  w_\text{unord})}\\ \\
= \frac{w_p(P) w_p(P) + (\lambda w_\text{ord} + (1-\lambda)  w_\text{unord})}
{w_p(P) w_p(P) + (\lambda w_\text{ord} + (1-\lambda)  w_\text{unord})} = 1
\end{array}
\]    
which concludes the proof.
\end{proof}

\begin{lemma}\label{lemma:symmetric-tve-equals-1}
\upshape{
For any $\alpha \in [0,+\infty)$,    $\text{Tve}^{\alpha,\alpha,\oplus^{\max}_{\simL}}(C_1,C_2) = 1$ if and only if $C_1 = C_2$.
}
\end{lemma}

\begin{proof}
We recall that
    \begin{center}
    $\text{Tve}^{\alpha,\alpha,\oplus^{\max}_{\simL}}(C_1,C_2) = \dfrac{A}{A + \alpha B + \alpha C}, ~\text{where:}$
\end{center}
{\small
\begin{align*}
A &= \frac{1}{2} \left( \sum_{x \in C_1} \oplus^{\agg}_{\simL}(x,C_2) + \sum_{y \in C_2} \oplus^{\max}_{\simL}(y,C_1) \right) \\
B &= \sum_{x \in C_1} \left( 1 - \oplus^{\max}_{\simL}(x,C_2) \right) \\
C &= \sum_{y \in C_2} \left( 1 - \oplus^{\max}_{\simL}(y,C_1) \right) 
\end{align*}
}

Now, let us first assume that $C_1 = C_2$. We need to consider two cases: either $\simL$ is a flat literal similarity measure or it is a weighted literal similarity measure. We start with the former.

If $l \in C$, then obviously $\oplus^{\max}_{\simL}(l,C) = 1$ because the flat literal similarity between a literal and itself is equal to $1$ (Lemma~\ref{lemma:flat-literal-simi-equals-1}). So both $B$ and $C$ simplify to $0$, and thus $\text{Tve}^{\alpha,\alpha,\oplus^{\max}_{\simL}}(C_1,C_1) = \frac{A}{A} = 1$.

Now, consider the case where $\simL$ is a weighted similarity measure. Reasoning is similar to the case of flat literal similarity, and from Lemma~\ref{lemma:weighted-literal-simi-equals-1} we can deduce that $\oplus^{\max}_{\simL}(l,C) = 1$ when $l \in C$, which allows to simplify the computation of $B$ and $C$ to obtain $0$, and finally $\text{Tve}^{\alpha,\alpha,\oplus^{\max}_{\simL}}(C_1,C_1) = \frac{A}{A} = 1$.

Now, let us assume that $\text{Tve}^{\alpha,\alpha,\oplus^{\max}_{\simL}}(C_1,C_2) = 1$. Reasoning towards a contradiction, assume that $C_1 \neq C_2$. Without loss of generality, we consider that there is a literal $l \in C_1$ such that $l \not\in C_2$. This means that $B = \sum_{x \in C_1} (1 - \oplus^{\max}_{\simL}(x,C_2)) > 0$ (because $\oplus^{\max}_{\simL}(x,C_2)) < 1$). This implies that $A + \alpha B + \alpha C > A$ (even if $C = 0$), and thus $\text{Tve}^{\alpha,\alpha,\oplus^{\max}_{\simL}}(C_1,C_2) < 1$, which is a contradiction, so we deduce that $C_1 = C_2$. This reasoning holds for both the flat and weighted literal similarities, so this concludes the proof.
\end{proof}

\begin{lemma}\label{lemma:identical-sets-of-clauses}
Let $\bM \in \{\bM^{x,w_g}_{sb,\lambda}, \bM^{x,w_g}_{eq,\lambda}\}$ be a similarity model, and $\Phi$ be a set of clauses. $\simS(\Phi, \Phi) = 1$ holds.
\end{lemma}

\begin{proof}
Recall how similarity is computed. For the similarity models considered, $\simS^{\bm}$ is used to compare sets of clauses:
\[
\simS^{\bm}(\Phi, \Psi) = \frac{
\sum\limits_{(C_1, C_2) \in \bm(\Phi, \Psi)}
 w_g(C_1, C_2) \times \simC(C_1, C_2)
}{
\sum\limits_{(C_1, C_2) \in \bm(\Phi, \Psi)}
w_g(C_1, C_2)
}
\]

where $\bm(\Phi, \Psi) = $
\[
\left\{
\begin{aligned}
(C_1, C_2) \mid\, & \big(C_1 \in \Phi,\, C_2 = \mathtt{argmax}_{C' \in \Phi} \simC^{\text{flat}}(C_1, C')\big) \\
\text{or } &  \big(C_1 \in \Psi,\, C_2 = \mathtt{argmax}_{C' \in \Phi} \simC^{\text{flat}}(C_1, C')\big)
\end{aligned}
\right\}
\]
$w_g(C_1, C_2) = \agg( w_{c}(C_1),w_{c}(C_2))$ with $\agg$ an aggregation function. Let us study what happens when $\Phi = \Psi$. First of all, $\bm$ simplifies into $\bm(\Phi, \Phi) = \{(C_1, C_2) \mid C_1 \in \Phi,\, C_2 = \mathtt{argmax}_{C' \in \Phi} \simC^{\text{flat}}(C_1, C')\}$. Since $\simC^{\text{flat}}$ is a symmetric fuzzy Tversky similarity measure (Definition~\ref{def:lit-flat-similarity}), $\simC^{\text{flat}}(C_1, C') = 1$ if and only if $C_1 = C'$ (Lemma~\ref{lemma:symmetric-tve-equals-1}). So, the simplification of $\bm$ continues: $\bm(\Phi, \Phi) = \{(C_1, C_1) \mid C_1 \in \Phi\}$, and then 
\[
\simS^{\bm}(\Phi, \Phi) = \frac{
\sum\limits_{C_1 \in \Phi}
 w_g(C_1, C_1) \times \simC(C_1, C_1)
}{
\sum\limits_{C_1 \in \Phi}
w_g(C_1, C_1)
}
\]
and again, $\simC$ is a fuzzy Tversky similarity measure (Definition~\ref{def:lit-flat-similarity}) so $\simC(C_1,C_1) = 1$, thus $\simS^{\bm}(\Phi, \Phi) = 1$.
\end{proof}

\begin{lemma}\label{lemma:identical-sets-of-clauses-eq}
Let $\bM^{x,w_g}_{eq,\lambda}$ be a similarity model, and $\Phi,\Psi $ be two sets of clauses. If $\simS(\Phi, \Psi) = 1$ then $\Phi = \Psi$.
\end{lemma}

\begin{proof}
We assume that $\simS(\Psi,\Psi) = \simS^{\bm}(\Phi, \Psi) = 1$. Observing the fraction that defines $\simS^{\bm}$, this implies that $\simC(C_1,C_2) = 1$ for all $(C_1,C_2) \in \bm(\Phi,\Psi)$. 

Reasoning towards a contradiction, we assume that $\Phi \neq \Psi$, {\em i.e.} (without loss of generality) there is a clause $C_1 \in \Phi$ such that $C_1 \not\in \Psi$. In the computation of $\bm(\Phi,\Psi)$, there is thus a pair of clauses $(C_1,C_2)$ such that $C_1 \neq C_2$, which implies (from Lemma~\ref{lemma:symmetric-tve-equals-1}) that $\simC(C_1,C_2) < 1$. This leads to $\simS^{\bm}(\Phi, \Psi) \neq 1$, hence a contradiction, and this concludes the proof.
\end{proof}

\begin{proof}[\textbf{Proof (Proposition~\ref{prop:compatibility})}]
We know that $\bM^{x,w_g}_{eq,\lambda} = \tuple{\simP^{eq}, \simL, \simC, \simS}$ satisfies all principles except Non-Zero, since, having no similarity elements can lead to a global similarity of $0$, even if there is some partial local similarity, due to importance weight focusing on the dissimilarity. 
However, consider a slightly modified model $\bM_{NZ}$, identical to $\bM^{x,w_g}_{eq,\lambda}$ in all respects, but with the constraint that all local similarities must be strictly positive: i.e., for all object pairs considered at any level, if $\similarity(x_i, y_j) > 0$, then they are assigned non-zero weights, and all weights over the domain of comparison are strictly positive. 
Under this construction, any local similarity value greater than $0$ will necessarily propagate a strictly positive contribution through the weighted aggregation steps. Thus, the global similarity will also be greater than $0$, satisfying the Non-Zero principle. 
Consequently, the modified model $\bM_{NZ}$ satisfies all the axioms simultaneously, proving their compatibility.
\end{proof}

\begin{proof}[\textbf{Proof (Theorem \ref{theo:=1})}]
    We start by considering both families of similarity models $\bM^{x,w_g}_{\rho,\lambda}$ with $\rho \in \{sb, eq\}$. Let $A, B$ be two arguments such that $A \approx B$. Recall that this means there are two bijections $f, f'$ (respectively between the supports of $A$ and $B$, and between the claims of $A$ and $B$) such that $\forall \phi \in \supp(A)$, $\phi \equiv f(\phi)$, and $\forall \psi \in \conc(A)$, $\psi \equiv f'(\psi)$. This means that, when the supports (respectively claims) of $A$ and $B$ are compiled into CNF, we obtain exactly the same set of clauses. This means that  $\similarity\mathtt{Arg}^{\bM}_{\eta}(A, B) =$
\begin{center}
$\eta \times \simS(\supp(A^{\mathtt{c}}), \supp(A^{\mathtt{c}})) + (1-\eta) \times \simS(\conc(A^{\mathtt{c}}), \conc(A^{\mathtt{c}})).$    
\end{center}
From Lemma~\ref{lemma:identical-sets-of-clauses}, we know that $\simS(\supp(A^{\mathtt{c}}), \supp(A^{\mathtt{c}})) = 1$ and $\simS(\conc(A^{\mathtt{c}}), \conc(A^{\mathtt{c}})) = 1$, so $\similarity\mathtt{Arg}^{\bM^{x,w_g}_{\rho,\lambda}}_\eta(A,B) = 1$ for $\rho \in \{sb, eq\}$.

Now, we assume that $\similarity\mathtt{Arg}^{\bM^{x,w_g}_{eq,\lambda}}_\eta(A,B) = 1$. From Definition~\ref{def:smArguments}, we deduce that $\simS(\supp(A^{\mathtt{c}}), \supp(B^{\mathtt{c}})) = 1$ and $\simS(\conc(A^{\mathtt{c}}), \conc(V^{\mathtt{c}})) = 1$. With the similarity model $\bM^{x,w_g}_{eq,\lambda}$, the only way to obtain a similarity of $1$ between two sets of clauses is if these sets of identical (Lemma~\ref{lemma:identical-sets-of-clauses-eq}), which implies that $\supp(A^{\mathtt{c}}) = \supp(B^{\mathtt{c}})$ and $\conc(A^{\mathtt{c}}) = \conc(B^{\mathtt{c}})$, and thus $A \approx B$, which concludes the proof.
\end{proof}

\end{document}